\documentclass[12pt]{amsart}

\usepackage{amssymb,latexsym,amsmath,extarrows, mathrsfs, amsthm}
\usepackage{graphicx}
\usepackage{bm}
\usepackage{float,url}

\usepackage[margin=1in, centering]{geometry}

\renewcommand{\theequation}{\arabic{section}.\arabic{equation}}

\newtheorem{lem}{Lemma}[section]

\newtheorem{thm}[lem]{Theorem}

\newtheorem{cor}[lem]{Corollary}
\theoremstyle{definition}

\newtheorem{theorem}{Theorem}[section]

\newtheorem{corollary}[theorem]{Corollary}

\newcommand{\tr}{\text{tr}\,}      	    
\def\benm{\begin{enumerate}}            
	\def\eenm{\end{enumerate}}              

\newcommand{\tb}{\textbf}
\newcommand{\al}{\alpha}
\newcommand{\be}{\beta}

\usepackage{xr}

\begin{document}

\title{Transport Model for Feature Extraction}
\author{Wojciech Czaja}
\address{CSCAMM and Department of Mathematics\\
University of Maryland\\
College Park, MD 20742, USA}
\email{wojtek@math.umd.edu}

\author{Dong Dong}
\address{CSCAMM\\
University of Maryland\\
College Park, MD 20742, USA}
\email{ddong12@cscamm.umd.edu}

\author{Pierre-Emmanuel Jabin}
\address{CSCAMM and Department of Mathematics\\
University of Maryland\\
College Park, MD 20742, USA}
\email{pjabin@umd.edu}

\author{Franck Olivier Ndjakou Njeunje}
\address{Department of Mathematics\\
University of Maryland\\
College Park, MD 20742, USA}
\email{fndjakou@math.umd.edu}

\begin{abstract}
	We present a new feature extraction method for complex and large datasets, based on the concept of transport operators on graphs. The proposed approach generalizes and extends the many existing data representation methodologies built upon diffusion processes, to a new domain where dynamical systems play a key role. The main advantage of this approach comes from the ability to exploit different relationships than those arising in the context of e.g., Graph Laplacians. Fundamental properties of the transport operators are proved. We demonstrate the flexibility of the method by introducing several diverse examples of transformations. We close the paper with a series of computational experiments and applications to the problem of classification of hyperspectral satellite imagery, to illustrate the practical implications of our algorithm and its ability to quantify new aspects of relationships within complicated datasets.
\end{abstract}

\maketitle


\let\thefootnote\relax
\let\thefootnote\relax
\let\thefootnote\relax\footnote{ WC was partially supported by  LTS grants DO 0048-0049-0050-0052 and D00030014. DD was partially supported by  LTS grant DO 0052. PEJ was partially supported by NSF Grant 161453, NSF Grant RNMS (Ki-Net) 1107444 and by LTS grants DO 0048-0049-0050-0052 and D00030014. FNN was partially supported by LTS grants DO 0048-0049-0050 and 0052.}

\section{Introduction} \label{sec: introduction}

Feature extraction has been at the core of many data science applications for more than a century. The goal of feature extraction is to derive new measurements (or features) from an initial set of measure data with the intention of retaining the core information while eliminating redundancies. A well-known feature extraction algorithm is principal components analysis (PCA) which can be traced back to the year 1901~\cite{pearson1901lines}. However, due to the linear nature of PCA, the method falls short in capturing the intrinsic structure of the data when a non-linear relationship governs the underlying structure within the data. Since then, the complex, non-linear, and growing amount of data have led scientists to come up with new techniques. A few well-known techniques are: kernel PCA~\cite{scholkopf1997kernel}, isomap~\cite{tenenbaum2000global}, locally linear embedding (LLE)~\cite{roweis2000nonlinear}, and Laplacian eigenmaps (LE)~\cite{belkin2002laplacian}. Today, the use of feature extraction techniques varies based on applications from the classification of hyperspectral images~\cite{benedetto2009frame,sun2014nonlinear,sun2014ul,zhao2016spectral} to the prediction of stock market prices~\cite{zhong2017forecasting}.

The aforementioned non-linear feature extraction methods lead to applications of linear operators, e.g., the Laplacian. In the present study, we have developed a more general approach that constructs non-linear feature extraction algorithms based on non-linear operators, such as appropriately chosen transport by advection operators. A recent technique~\cite{gerber2017multiscale} sought to find the optimal transport method between two point sets based on an adaptive multiscale decomposition, which itself is derived from diffusion wavelets and diffusion maps. In our work, we focus on the transport operator directed by by velocity fields ~\cite{bennett2012transport,hundsdorfer2013numerical,vreugdenhil1993numerical}, because of its well-studied properties as well as its partial similarity to Schroedinger Eigenmaps method~\cite{czaja2013schroedinger}. This transport model has not been used in the literature as a tool for building a feature extraction algorithm. Nevertheless, some related work can be found in the fields of water resource management and in bio-medical research~\cite{hansen2016reduced}, where feature extraction is used to construct simplified transport models for cardiovascular flow.

At its core, our work will focus on exploring and exploiting the differences and similarities of this novel approach to the state-of-the-art feature extraction algorithms found in the literature. After providing some background in Section~\ref{sec: background}, we introduce the model in Section~\ref{sec:transport} together with some properties of the model. The algorithm for our approach is given in Section~\ref{sec:transport_algorithm}, and we provide an application of our algorithm for feature extraction and subsequent classification of hyperspectral image data in Section~\ref{sec:hyperspectral_exp} and Section~\ref{sec:results}. Some open problems are posed in the last section.

\section{Background} \label{sec: background}

In many data science applications, high dimensional data tend to lie on low dimensional manifolds within the high dimensional space. To take advantage of this information, methods such as the Laplacian eigenmaps (LE)~\cite{belkin2002laplacian} and the Schroedinger eigenmaps (SE)~\cite{czaja2013schroedinger}, invoke the adjacency graph constructed from a set of initial points, $X = \lbrace \mathbf{x}_{1},\mathbf{x}_{2},\ldots,\mathbf{x}_{n} \rbrace$ in $\mathbb{R}^{d}$, in order to extract the most important features from the bunch.

In LE, 
the problem is reduced to solving the following generalized eigenvector problem (after building a weighted graph $G$ from the $n$ points),
\begin{eqnarray} \label{eq: LE Gen eig problem}
	L\, \mathbf{f} = \lambda\, D \,\mathbf{f},
\end{eqnarray}
where $L = D-W$, viz., the Laplacian matrix, with $W$ representing the (symmetric) weight matrix $(w_{ij})$ and $D$ the diagonal matrix with entries $d_{ii} = \sum_{j} w_{ij}$.
Let $\lbrace \mathbf{f}_{0},\mathbf{f}_{1},\ldots,\mathbf{f}_{n-1} \rbrace$ be the solution set to (\ref{eq: LE Gen eig problem}) written in ascending order according to their eigenvalues. The $m$-dimensional Euclidean space mapping is given by
\[ \mathbf{x}_{i} \rightarrow [\mathbf{f}_{1}(i), \mathbf{f}_{2}(i), \ldots, \mathbf{f}_{m}(i)]. \]

In SE, 
the $m$-dimensional Euclidean space mapping is given in a similar manner. Dubbed as a generalization of the LE algorithm, SE uses partial knowledge about the data set $X$ and fuses this information into the LE algorithm to obtain better representation or more desirable results. Additional work related to data fusion can be found in the following papers~\cite{benedetto2012semi,cloninger2017pre, doster2014harmonic, halevy2011extensions}. The problem in SE is reduced to solving the following generalized eigenvector problem,
\begin{eqnarray} \label{eq: Sch Gen eig problem}
	S\, \mathbf{f} = \lambda\, D\, \mathbf{f},
\end{eqnarray}
where $S = L + \alpha V$, viz., the Schroedinger matrix, with $V$ as the potential matrix encoding the partial information and $\alpha$ as a real parameter keeping the balance between the matrices $L$ and $V$.

The algorithm we are developing in this article, viz., transport eigenmaps (TE), has some similarities to SE in the sense that both algorithms use extra information about the data set to define a generalization of LE. 
Unlike supervised learning techniques which assume prior knowledge of the ground truth, i.e., knowledge of what the output values for our samples should be, SE and TE only assumes partial knowledge of said ground truth. This puts SE and TE in a class of machine learning techniques between supervised learning and unsupervised learning (no prior knowledge) called semi-supervised learning (see \cite{belkin2004semi,davidson2009knowledge,fang2017hyperplane, wagstaff2001constrained, zhang2007semi} for more examples).
While SE uses potentials to encode to additional information, TE may use advection (the active transportation of a distribution by a flow field) or measure/weight modifiers. In contrast to SE, TE could come from a non-linear operator which we will describe in section~\ref{sec:transport}. 
\section{The transport model} \label{sec:transport}
Transport operators have been used in modeling and analyzing data in a variety of fields \cite{bartsch1989modelling,brogniez2009,Brutsaert1979evap,Guo2014flow,Shen1996visu, sibert1999fish, Song2003expe}. We aim to bring this idea into the graph setting to help with data representation.  
\subsection{Notation and introduction}
We first briefly present the basic setting for studying transport model on graphs. Fix a weighted simple graph $G$ with $n$ nodes. Let $v$ be a function defined on the edges of $G$. As an $n\times n$ matrix, $v$ is assumed to be anti-symmetric since it will be used to model a velocity field. The transport operator $T$ acting on a vector $\tb y$ is formally defined as
\begin{equation} \label{eq: def of T}
T\,\tb y=L\,\tb y-\text{div}(v\tb y).
\end{equation}
This agrees with the continuous transport operator.

We use the following rules to translate between the continuous and discrete settings. For any matrix $A$, which is viewed as a function defined on the edges, the divergence of $A$ is a function defined on nodes, i.e., is a vector:
\begin{equation}
\text{div}(A)_i:=\sum_j A_{ij}.
\end{equation}
When $A$ models a velocity field on the graph, $\text{div}(A)_i$ is just the net flow coming from or to the node $i$. 

For any function $f$ defined on the nodes, its gradient, the dual operator of the divergence, is defined on the edges
\begin{equation}
(\nabla f)_{ij}:=(f_j-f_i)w_{ij}.
\end{equation}

A matrix $A$ (e.g., a velocity field) can act on a vector $f$ (e.g., a probability distribution) in the following way
$$(Af)_{ij}=(fA)_{ij}:=\frac{f_i+f_j}{2}A_{ij}.$$
This corresponds to the standard centered discretization of the transport operator (after taking the divergence).
 
The Laplacian of $f$, $\Delta f:=\text{div}(\nabla f)$, is defined on the nodes:
\begin{equation}
(\Delta f)_i=\sum_j (f_j-f_i)w_{ij}.
\end{equation}
This differs from the graph Laplacian $L$ by a sign, as we prefer to have positive semi-definite graph Laplacian.

Based on the above rules, we have $(v\tb y)_{ij}=v_{ij}\frac{y_i+ y_j}{2}$ and $\text{div}(v\tb y)=\sum_j(v\tb y)_{ij}=\frac{1}{2}\sum_j ( y_i+ y_j)v_{ij}$. Therefore, the definition of $T$ \eqref{eq: def of T} becomes
\begin{equation} \label{eq: ori transport}
(T\,\tb y)_i=\sum_j ( y_i- y_j)w_{ij}-\sum_j ( y_i+y_j)\frac{v_{ij}}{2}.
\end{equation}
It is unclear from the expression \eqref{eq: ori transport} that a transport operator $T$ would always produce real eigenvalues as the Laplacian and Schroedinger operators do. We will address this issue in the next subsection.

\subsection{Self-adjointness}
 As the properties of the transport operator
ultimately depend on $v$, an anti-symmetric matrix, we aim to find $v$'s so that the corresponding transport operator $T$ is self-adjoint (probably with respect to a non-standard inner product). 

For any positive definite matrix $A$, we use $\langle , \rangle_A$ to denote the inner product
\begin{equation*}
\langle \tb y,\,\tb z\rangle_A:=\tb y^tA\tb z.
\end{equation*}
When $A$ is the identity matrix, this agrees with the standard inner-product. 

It is natural to assume that $v_{ij}=0$ if the nodes $i$ and $j$ are not connected, as $w_{ij}=0$ in this case as well. Let 
$$
\bar v_{ij}:=\frac{v_{ij}}{2w_{ij}} \text{ if $i$ and $j$ are connected},
$$
and $\bar v_{ij}=0$ otherwise. Then $\bar v$ is also anti-symmetric, $\frac{v_{ij}}{2}=\bar v_{ij}w_{ij}$, and
\begin{align*}
(T\,\tb y)_i&=\sum_j (y_i-y_j)\,w_{ij}-\sum_{j}(y_i+y_j)\,\bar v_{ij}\,w_{ij}\\
&=\sum_j [(1-\bar v_{ij})\,y_i-(1+\bar v_{ij})\,y_j]\,w_{ij}
\end{align*}

Our goal is to find a positive definite matrix $X$ such that $T$ is self-adjoint with respect to $\langle, \rangle_X$. It turns out $\bar v_{ij}=\frac{a_j-a_i}{a_j+a_i}$ is a natural choice (see the Supplementary Materials for a discussion and comparison with another choice $\bar v_{ij}=a_j-a_i$). 
\begin{thm} \label{main thm}
	Let $W=(w_{ij})$ be a symmetric matrix. Assume $\bar v_{ij}=\frac{a_j-a_i}{a_j+a_i}$ for some positive $a_i$'s. Then the operator $(T\, \tb y)_i=\sum_j[y_i-y_j-\bar v_{ij}(y_i+y_j)]w_{ij}$ is self-adjoint with respect to the inner product $\langle,\rangle_X$, with $X=\text{diag}(ca_i)$ for some positive $c$.
\end{thm}
\begin{proof}
	For the convenience of future discussion, denote $X=\text{diag}(x_i)$ and we try to ``solve'' for $x_i$. In general, $X$ could be non-diagonal. We need to verify that for any vectors $\tb y$ and $\tb z$
	\begin{equation} \label{goal of self ajoint}
	\sum_i (T\,\tb y)_i \,z_i \, x_i=\sum_iy_i\, (T\,\tb z)_i \, x_i
	\end{equation}
	The left-hand-side (LHS) of \eqref{goal of self ajoint} is
	\begin{align*}
		\sum_i (T\,\tb y)_i\, z_i\,x_i&=\sum_{i,j}[(1-\bar v_{ij})\,y_i-(1+\bar v_{ij})\, y_j]\,z_i \, x_i\, w_{ij}\\
		&=\sum_{ij}[(1-\bar v_{ij})\,y_i\,z_i\,x_i-(1-\bar v_{ij})\,y_i\,z_j\,x_j]\,w_{ij}\\
		&=\sum_i y_i\,\sum_j [(1-\bar v_{ij})\,z_i\,x_i-(1-\bar v_{ij})\,z_j\,x_j]\,w_{ij}
	\end{align*}
	Compare this with the right-hand-side (RHS) of \eqref{goal of self ajoint}
	\begin{equation*}
		\sum_iy_i\,(T\tb z)_i\,x_i=\sum_i y_i\,\sum_j [(1-\bar v_{ij})\,z_i\,x_i-(1+\bar v_{ij})\,z_j\,x_i]\,w_{ij},
	\end{equation*}
	and we see that in order to make \eqref{goal of self ajoint} hold,
	\begin{equation} \label{condition}
	(1-\bar v_{ij})\,x_j=(1+\bar v_{ij})\,x_i
	\end{equation}
	must be true for any pair of connected nodes $i$ and $j$. 
	
	Now make use of the assumption $\bar v_{ij}=\frac{a_j-a_i}{a_j+a_i}$. In this case, the key condition \eqref{condition} becomes
	\begin{equation*}
		\left(1-\frac{a_j-a_i}{a_j+a_i}\right)\, x_j=\left(1+\frac{a_j-a_i}{a_j+a_i}\right)\, x_i,
	\end{equation*}
	which is 
	\begin{equation*}
		a_i\, x_j=a_j\, x_i.
	\end{equation*}
	This clearly holds as $x_i=ca_i$ by the assumption of the theorem. 
\end{proof}

We can immediately extend this theorem to a more general model by introducing a symmetric matrix $r$. This new collection of parameters will allow us to implement the transport eigenmap method in various settings.
\begin{thm} \label{coro about Tvr}
	Let $r=(r_{ij})$ and $W=(w_{ij})$ be symmetric matrices. Define $T_v^r$ to be the operator such that
	\begin{equation}
	(T_v^r\,\tb y)_i=\sum_j[r_{ij}\,(y_i-y_j)-\bar v_{ij}\,(y_i+y_j)]\,w_{ij}.    
	\end{equation}
	Assume $\bar v_{ij}=\frac{a_j-a_i}{a_j+a_i}\,r_{ij}$ for some positive $a_i$'s. Then $T_v^r$ is self-adjoint with respect to the inner product $\langle,\rangle_X$, with $X=\text{diag}(ca_i)$ for some positive $c$.
\end{thm}
\begin{proof}
     Simply notice that the symmetric matrix $r$ can be incorporated into the symmetric matrix $W$ and thus the operator $T_v^r$ has the same form as $T$ in Theorem \ref{main thm}.
\end{proof}

When $\bar v_{ij}=\frac{a_j-a_i}{a_j+a_i}\,r_{ij}$, the general transport operator $T_v^r$ can be rewritten as 
\begin{equation} \label{def: general transport}
	(T_v^r\,\tb y)_i=\sum_j \left(\frac{2a_i}{a_i+a_j}\,y_i-\frac{2a_j}{a_i+a_j}\,y_j\right)\,w_{ij}=\sum_j(a_iy_i-a_jy_j)\,w_{ij}\,\frac{2r_{ij}}{a_i+a_j}.
\end{equation}
This expression also indicates that $T_v^r$ is non-negative when $\bar v_{ij}=\frac{a_j-a_i}{a_j+a_i}\,r_{ij}$.
\begin{thm} \label{thm: nonnegative}
	The operator defined by \eqref{def: general transport} is non-negative in $\ell^2_{X}$, where $X=\text{diag}(ca_i)$ for some positive $c$. More precisely, 
	\begin{equation} \label{eq: nonnegative of T}
		\langle \tb y,\; T_v^r\,\tb y\rangle_X=\frac{c}{2}\sum_{i,j}(\tilde{y}_i-\tilde{y}_j)^2\,\tilde{w}_{ij}\ge 0,
	\end{equation}
    with $\tilde{w}_{ij}:=w_{ij}\,\frac{2r_{ij}}{a_i+a_j}$ and $\tilde{y}_i:=a_i\,y_i$. In particular, $T_v^r\,\tb y=0$ iff the quantity $a_iy_i$ is constant on every connected component of the graph. 
\end{thm}
\begin{proof}
	By a straightforward computation,
	$$
	\langle \tb y,\; T_v^r\,\tb y\rangle_X=c\sum_i y_i\,a_i\,(T_v^r\,\tb y)_i=c\sum_{i,j}\tilde{y}_i\,(\tilde{y_i}-\tilde{y}_j)\,\tilde{w}_{ij}=\frac{c}{2}\sum_{i,j}(\tilde{y}_i-\tilde{y}_j)^2\,\tilde{w}_{ij}\ge 0.
	$$
	When $T_v^r\tb y=0$, the above expression is $0$ and thus $\tilde{y}_i$ must the constant on any connected component. The converse is trivial by \eqref{def: general transport}.
\end{proof}

The above theorem ensures that $T_v^r$ is diagonalizable, with real-valued and negative eigenvalues. In applications, we will however look for the generalized eigenvectors of $T_v^r$: eigenvectors that are normalized by the degree on the graph, {\em i.e.} vectors $\tb u$ s.t.
\[
T_v^r\, \tb u=\lambda\,D\,\tb u,
\]
where $D$ is the degree matrix as before: $d_{ii}=\sum_{j} w_{ij}$. Equivalently we are looking for the eigenvectors $\tb y$ of $D^{-1/2}\,T_v^r\,D^{-1/2}$ with $\tb y=D^{1/2}\,\tb u$ or $\tb u=D^{-1/2}\,\tb y$ and the same generalized eigenvalues. From Theorem \ref{thm: nonnegative}, it is now straightforward to deduce that
\begin{cor} \label{spectralcor}
		Let $T_v^r$ be given by \eqref{def: general transport} and let $D$ be the degree matrix. Then the operator $D^{-1/2}\,T_v^r\,D^{-1/2}$ is self-adjoint in $\ell^2_X$ and non-negative, where $X=\text{diag}(ca_i)$ for some positive $c$. Furthermore, $D^{-1/2}\,T_\mu\,D^{-1/2}\,\tb u=0$ iff $(D^{-1/2}\,\tb u)_ia_i$ is constant on connected components of the graph.
\end{cor}
\begin{proof}
	$D$ is self-adjoint on $\ell^2_X$, simply because $D$ is diagonal and so is the metric provided by $\langle, \rangle_{X}$. It would be very different if we had to use non-diagonal metric (and we would have to study directly $D^{-1/2}\,T_v^r\,D^{-1/2}$ instead of $T_v^r$).
	
	The operator $D^{-1/2}\,T_v^r\,D^{-1/2}$ is still non-negative with
	\[
	\langle \tb u,\;D^{-1/2}\,T_v^r\,D^{-1/2}\,\tb u \rangle_{X}= \langle D^{-1/2}\,\tb u,\;T_v^r\,D^{-1/2}\,\tb u \rangle_{X}\geq 0,
	\]
	and by Theorem \ref{thm: nonnegative}, equality holds iff $(D^{-1/2}\,\tb u)_ia_i$ is constant on connected components of the graph.
\end{proof}

Compared with the Laplacian operator $(L\tb y)_i=\sum_j (y_i- y_j)w_{ij}$, we see that $T_v^r$ generalizes $L$ in the following ways:
\begin{itemize}
	\item $a_i$ modifies the measure/coordinate and thus makes the representation of $i$-th point closer to the origin if $a_i$ is large or further away from the origin if $a_i$ is small.
	\item $r_{ij}$ can enlarge or reduce the weight $w_{ij}$ between two nodes $i$ and $j$, serving as a weight modifier.
\end{itemize}
We can then use these two sets of parameters to guide data representation given by LE.

\subsection{Two examples}
We will use TE to denote the general transport operator \eqref{def: general transport}. Although the matrix $r$ can be used to fuse extra information, the implementation with $r$ could be more time-consuming as the size of $r$ is $n^2$. We will therefore first look at two examples (denoted by TA and TG respectively) using $a_i$ only. As Section \ref{sec:results} will show, TA and TG are often good enough to handle classification tasks when one class is known. The general TE, however, is needed when more than one classes are known.
\subsubsection{Transport by advection (TA)}
Advection is the active transportation of a distribution by a flow field. 
Let $\bm{\mu}=[\mu_1, \mu_2, \dots, \mu_n]^t$ be a vector that will be used to direct the clustering process. Let $\beta$ be a real parameter. Set $a_i=1+\beta\mu_i$, $r_{ij}=(a_j+a_i)/2$, and $\bar v_{ij}=(a_j-a_i)/2$. Clearly $\bar v_{ij}=\frac{a_j-a_i}{a_j+a_i}\,r_{ij}$. By Theorem \ref{coro about Tvr}, the operator $T_{\bm \mu}:=T_v^r$ with
\begin{equation} \label{operator T1}
	(T_{\bm\mu}\,\tb y)_i=\sum_j [(1+\beta\mu_i)\,y_i-(1+\beta\mu_j)\,y_j]\,w_{ij}
\end{equation}
is self-adjoint and enjoys other desired properties.

The operator $T_{\bm \mu}$ can also be derived directly from the general operator $T$ \eqref{eq: ori transport} by choosing the velocity field $v=\beta \nabla\,\tb y$, $\beta\in\mathbb{R}$. In this case, $v_{ij}=\beta\,(y_j- y_i)\,w_{ij}$ and $T$ becomes
\begin{equation} \label{non linear T}
(T\,\tb y)_i=\sum_j (y_i- y_j)\,w_{ij}-\frac{\beta}{2}\,\sum_j (y_j^2-y_i^2)\,w_{ij},    
\end{equation}
which is no longer linear. We can then linearize the second term in \eqref{non linear T} in the direction of $\bm \mu$ and $T$ will be exactly $T_{\bm \mu}$ (see \cite{njeunje2018computational} for details).

This choice of operator is inspired by the porous medium equation, for which we refer for example to \cite{MR2286292} for a thorough discussion of this type of non-linear diffusion on $\mathbb{R}^d$. In the present context, the idea behind having $v(\tb y)=\beta\,\nabla\tb y$ is to use the distribution $\tb y$ itself to help with clustering. The velocity field $v(\tb y)$ naturally points in the direction of the higher values of $\tb y$ if $\beta<0$ or towards lower values if $\beta>0$. Similarly solving the advection-diffusion equation
\[
d_t\,\tb y+T\,\tb y=0,
\]
would naturally lead to concentration around higher values of $\tb y$ if $\beta<0$ (limited by the dispersive effects of the graph Laplacian) or {\em a contrario} to faster dispersion if $\beta>0$. The ability to control concentrations and hence clustering is of obvious interest for our purpose. 

\subsubsection{Transport by gradient flows (TG)}
Set $r_{ij} \equiv 1$ in \eqref{def: general transport}. Then the general transport operator $T_v^r$ becomes
\begin{equation} \label{T2}
	(T_v\,\tb y)_i=\sum_j(a_iy_i-a_jy_j)\,w_{ij}\,\frac{2}{a_i+a_j}.
\end{equation}

Note that this is in fact the same operator appeared in Theorem \ref{main thm}, where $v$ is an scaling-invariant gradient of $\tb a=[a_1,\dots,a_n]^t$. Here $a_i$ plays a similar role as $1+\beta \mu_i$ in the first example of the transport by advection. One advantage of having the extra term $\frac{2}{a_i+a_j}$ is that even the weight modifier $r$ is constant, the weight $w_{ij}$ could still be changed. In applications, the default value for the measure modifier $a_i$ is $1$ and some of them may be greater than $1$ if extra information is known. When $a_i\neq a_j$, which often indicates that the two points $i$ and $j$ belong to different clusters, the factor $\frac{2}{a_i+a_j}<1$, weakening the original weight $w_{ij}$. Therefore, the formulation of the operator $T_v$ achieves measure modification and weight modification simultaneously without using $r$.

\section{The transport eigenmap method\label{sec:transport_algorithm}}
We describe the implementation of our new TE (short for transport eigenmap or transport extended) algorithm, including TA and TG as two important special cases.

\subsection{The algorithm\label{sec:algorithm}}

The steps are identical to those of LE and SE. We only need to modify the matrix used in the generalized eigenvalue problem. Given a set of $n$ points $X = \lbrace \mathbf{x}_{1},\mathbf{x}_{2},\ldots,\mathbf{x}_{n} \rbrace$ in $\mathbb{R}^{d}$, the goal is to find a map
\[
\Phi:\ \mathbb{R}^d\longrightarrow \mathbb{R}^m,
\]
so that the $n$ points $ Y =\lbrace \mathbf{y}_{1},\mathbf{y}_{2},\ldots,\mathbf{y}_{n} \rbrace$ in $\mathbb{R}^{m}$ given by $\mathbf{y}_{i}=\Phi(\mathbf{x}_i)$ represents $\mathbf{x}_{i}$ for all $i$ from $1$ to $n$. 

The goal is typically to have a lower dimensional representation $Y$ of the set of points $X$ with $m\ll d$ while still keeping the main features of the original set $X$. For example
if the points lie on a $m$-dimensional manifold where $m \ll d$, the hope would be to take as map $\Phi$ a good approximation of the projection on the manifold. 


\begin{itemize}
	\item \textbf{Step 1:} Construct the adjacency graph using the $k$-nearest neighbor (kNN) algorithm. This is done by putting an edge connecting nodes $i$ and $j$ given that $\mathbf{x}_{i}$ is among the $k$ nearest neighbors of $\mathbf{x}_{j}$ according to the Euclidean metric. We choose $k$ large enough so that {\em the graph that we obtain is fully connected}.
	
	\item \textbf{Step 2:} Define the weight matrix, $W$, on the graph. The weights $w_{ij}$ in $W$ are chosen using the heat kernel with some parameter $\sigma$. If nodes $i$ and $j$ are connected,
	\[ w_{ij} = \exp\left(-\frac{\| \mathbf{x}_{i} - \mathbf{x}_{j} \|_2^{2}}{2 \sigma^{2}}\right);  \]
	otherwise, $w_{ij} = 0$. 
	
      \item \textbf{Step 3:} Choose an appropriate transport operator and construct the corresponding matrix. Recall the general transport operator given in \eqref{def: general transport}
      $$
	(T\,\tb y)_i=\sum_j(a_iy_i-a_jy_j)\,w_{ij}\,\frac{2r_{ij}}{a_i+a_j}.      
      $$
    Here, the vector $\tb a=[a_1,\dots,a_n]^t$ and the matrix $(r_{ij})$ are the parameters to be chosen.
    Let $W^r$ denote the matrix with entries $w^r_{ij}=w_{ij}\,\frac{2r_{ij}}{a_i+a_j}$. Then the matrix form of $T$ is 
    \begin{equation} \label{matrix form of TE}
    	T=\text{diag}(a_i\,\sum_j w^r_{ij})-W^r\text{diag}(a_i).
    \end{equation}
    To get the matrix form of the special operator TA, we can either set $a_i=1+\beta\mu_i$ and $r_{ij}=(a_i+a_j)/2$ in \eqref{matrix form of TE}, or use the operator form \eqref{operator T1} to derive its matrix form directly
    \begin{equation*}
    	TA=L(I+\beta\text{diag}(\mu_i)),
    \end{equation*}
	where $L=D-W$ is the Laplacian matrix and $I$ is the identity.
	
	Similarly, for the operator $TG$, we can let $r_{ij}=1$ in \eqref{matrix form of TE} or use the expression in Theorem \ref{main thm} to get
	\begin{equation*}
		TG=L-(D_v+Wv),
	\end{equation*}
    where $D_v=\text{diag}(\sum_j w_{ij}v_{ij})$, $W_v=(w_{ij}v_{ij})$ and $v_{ij}=(a_j-a_i)/(a_j+a_i)$.
    
	\item \textbf{Step 4:} Find the $m$-dimensional transport mapping $\Phi_T$ by solving the generalized eigenvector problem,
	\begin{eqnarray} \label{eq: Tra Gen eig problem}
	T\, \tb u = \lambda\, D\,\tb u,
	\end{eqnarray}
	This can be done because of Corollary \ref{spectralcor}.
	Denote $\lbrace u^0,u^1,\ldots,u^{n-1} \rbrace$ be the solution set to (\ref{eq: Tra Gen eig problem}) written in ascending order according to their eigenvalues. Since there is hence no additional information in $u^0$, we define the mapping $\Phi_T$ by
	\[
        \mathbf{x}_{i} \longrightarrow \Phi_T(\mathbf{x}_i)=[u^1_i, u^2_i, \ldots, u^m_i].
        \]
\end{itemize}
%
\subsection{A toy example\label{sec:controlled_sample}}
We illustrate the behavior of LE, SE and TE (including TA and TG) with a toy example. The first picture in Figure~\ref{fig:eig_stu_3C_Schrod} is a dataset with $500$ points. The ground truth is that there are $5$ clusters, each containing $100$ points. 

LE is an unsupervised method that preserves local distance. We chose $k=50$ for KNN in Step 1 and $\sigma=1$ in Step 2 for simplicity.

SE, which uses the matrix $S=L+\alpha V$, requires extra parameters: $\alpha\ge 0$ and the diagonal potential matrix $V$. Assume the red points are known. Simply let $V_i=1$ if the $i$-th point is red and $V_i=0$ otherwise. Let $\alpha = \hat{\alpha} \cdot \tr(L) / \tr(V)$. This new parameter $\hat{\alpha}$ will allow us to balance the impact of the Laplacian matrix $L$ and the potential $V$ in the algorithm. We chose $\hat{\alpha}=10$. As expected, points with non-zero potential (the red ones in this example) are pushed towards the origin. As $L$ tries to preserve local distance, other points close to the red are dragged towards the origin as well.

For TA, we chose $\beta=10$ and $\bm \mu$ in the same way as $V$: $\mu_i=1$ for red and $\mu_i=0$ for other points. The red go to the origin because of rescaling of the coordinates, but the surrounding points don't ``see'' any changes in distance. This explains the less dragging effect in TA compared with SE.

In TG, we set $a_i=1$ by default and $a_i=10$ for the red points. The red are even better separated from others. This is because the factor $\frac{2}{a_i+a_j}$ in \eqref{T2} is less than $1$ and thus weakens the original weight $w_{ij}$ if $i$ and $j$ are not both red.

For the general TE, the matrix $r$ needs to be determined. The default is $r_{ij}=1$. Then it is natural to set 
\begin{eqnarray} \label{choice of r}
	r_{ij} = 
	\begin{cases} 
		small (<1), & \mbox{if $i$ and $j$ belong to different clusters}, \\ 
		big (>1), & \mbox{if $i$ and $j$ belong to the same cluster}\\
		1, & \mbox{if unknown} 
	\end{cases}	
\end{eqnarray}
We set $small = 0.5$ and $big = 100$, which help further gathering the red points.

If the pre-identified cluster is not near the center of the data points, e.g., the blue points, then we can set $a_i$ to be less than $1$ for the blue to push them away from the origin. The weight modifier $r$ in TE is always helpful to gather these points to their natural location. See Figure~\ref{fig:eig_stu_3C_Schrod} for the case $a_i=0.5$ for the blue and $1$ otherwise in TE ($r$ remains to be in \eqref{choice of r}). 

The general TE can even handle the case when more than one cluster are known. Let $a_i=10$ for red and $a_i=0.5$ for blue. $r$ is still given by \eqref{choice of r}. We can see in Figure~\ref{fig:eig_stu_3C_Schrod} that both red and blue are well-separated from others.

\begin{figure}[H]
	\centering
	\includegraphics[scale = .30]{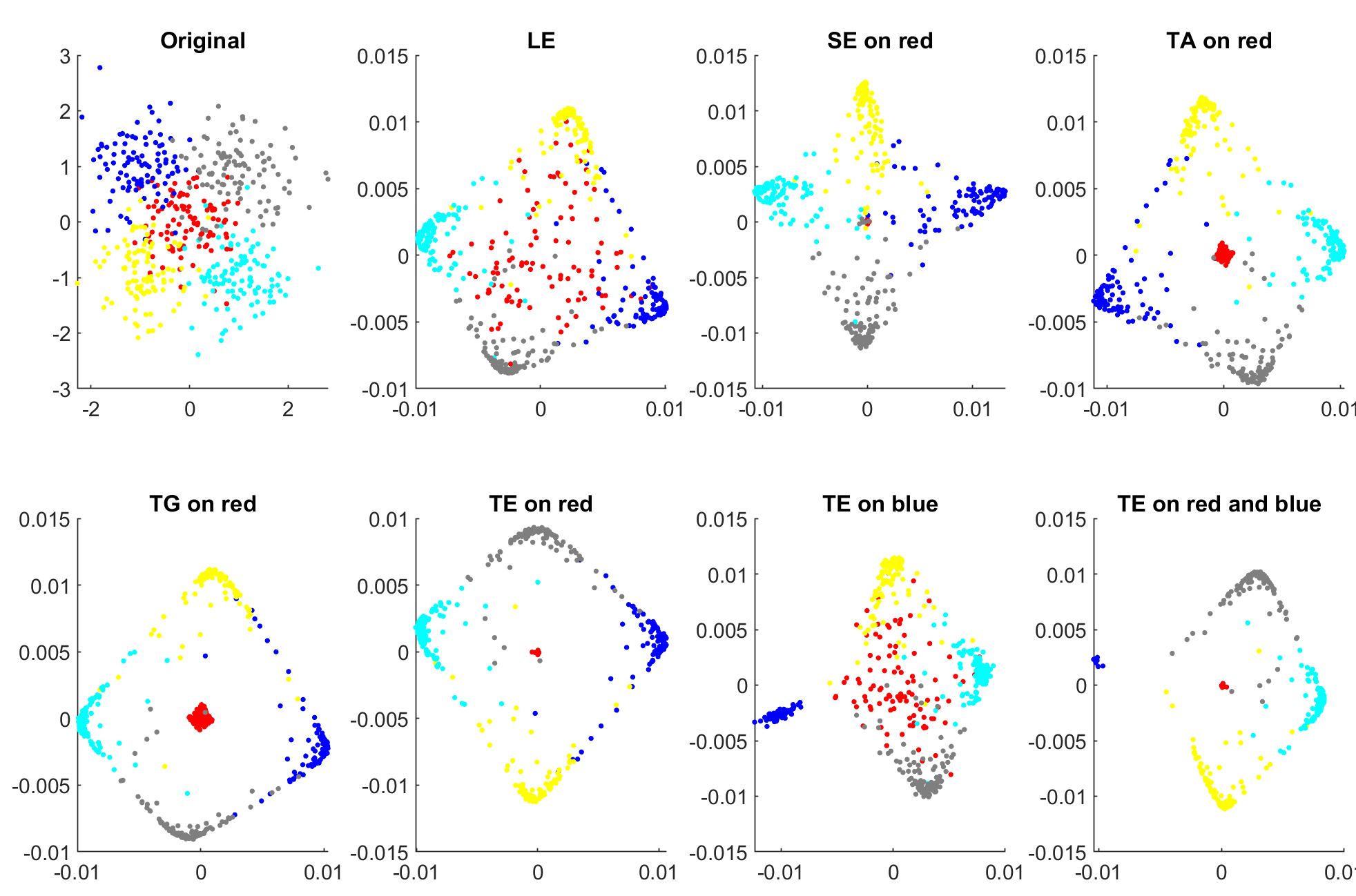}
	\caption{\footnotesize{The first plot presents the dataset, $500$  points grouped in $5$ clusters of $100$ points each. The next plots show the results of various mappings.}}
	\label{fig:eig_stu_3C_Schrod}
\end{figure}

It is very promising that TE can be used to help with clustering. We shall not pursuit this direction in this paper. Instead, we will test our methods on real hyperspectral data.


\section{Hyperspectral sample set experiments} \label{sec:hyperspectral_exp}

\subsection{The datasets}
We have taken advantage of two hyperspectral data sets: Indian Pines and Salinas. The Indian Pines dataset (cf. an example in Figure~\ref{fig:Indian_pines} in the supplementary document) was gathered by AVIRIS (Airborne Visible/Infrared Imaging Spectrometer) sensor over the Indian Pines test site in North-western Indiana. The Indian Pines dataset consists of $145\times145$ pixels images that contain $224$ spectral bands in the wavelength range $0.4~10^{-6}$ to $2.5~10^{-6}$ meters. The ground truth available is designated into sixteen classes (see Table~\ref{tab:pines} in the supplement). The number of bands has been reduced to $200$ by removing bands covering the region of water absorption. The Indian Pines dataset is available through Purdue's university MultiSpec site~\cite{hyperspectral,PURR1947}.

The Salinas dataset  was similarly gathered by AVIRIS sensor over Salinas Valley, California (see Figure~\ref{fig:Salinas} in the supplementary document). With again a similar structure, Salinas images are $512 \times 217$ pixels with $224$ spectral bands of approximately $3.7$ meter high spatial resolution. The ground truth available is also clustered into sixteen classes (see Table~\ref{tab:salinas} in the supplement). We again reduce the number of bands to $204$ by removing those bands covering the region of water absorption. The Salinas dataset is publicly available~\cite{hyperspectral}.

For easier testing purposes, we have also used a small sub-scene of the Salinas dataset, which we denote Salinas-B (shown in Figure~\ref{fig:SalinasB} in the supplement). Salinas-B consists of a $150 \times 100 \times 204$ data cube located within the same scene at  [samples, lines]=[200:349, 40:139] and includes only eight classes (see Table~\ref{tab:SalinasB} in the supplement). The Salinas-B dataset was used to allow for a faster and more thorough exploration of the parameters' space.

After the various mappings, we employ Matlab's 1-nearest neighbor algorithm to classify the data sets. We use $10\%$ of the data from each class to train the classifier and the remaining number of data points as the validation set. We took an average of ten runs to produce the confusion matrices, each using a disjoint set of data to train the classifier. 
\subsection{Choice of parameters\label{sec:parameter}}

Following the description of the mapping algorithms for the various methods under consideration in subsection \ref{sec:algorithm}, we made the following choices to construct the graph over which all methods rely
\begin{itemize}
\item The adjacency graph is built using $k=12$ nearest neighbors;
\item The weight matrix was obtained by using $\sigma = 1$;
  \item We calculated $m=50$ generalized eigenvectors for the Indian Pines dataset and $m=25$ for the Salinas-B dataset. The final mappings were obtained from those generalized eigenvectors as described in Step~4 of subsection \ref{sec:algorithm}.
  \end{itemize}
For SE, TA and TG, we also need to choose the potential $V$, the vector $\bm \mu$ and $\tb a$. In our testing, for example, we have assumed prior knowledge of either class 2-corn-notill or class 11-soybean-mintill in the Indian Pines dataset. This leads to the typical choice in the $11-soybean-mintill$ case  
\begin{eqnarray*}
	V_i,\;\mu_i = 
	\begin{cases} 
		1, & \mbox{if } x_i \in \mbox{Class  11--soybean-mintill}, \\ 
		0, & \mbox{elsewhere.} 
	\end{cases}	
\end{eqnarray*}
In TG, the default is $a_i=1$ and we will set $a_i=\beta$ for the known points. It remains to chose the parameters $\alpha$ and $\beta$. For SE, recall that in Section \ref{sec:controlled_sample} we introduced the parameter $\hat{\alpha}$ given by $\alpha = \hat{\alpha} \cdot \tr(\Delta) / \tr(V)$. To obtain the results listed in the next subsection, we used
\begin{itemize}
\item $\hat{\alpha} = 10^{4}$ for the Indian Pines data set and $\hat{\alpha} = 10^{2}$ for the Salinas-B data set for SE;
  \item $\beta = 20$ for both the Indian Pines and the Salinas-B data set for TA and TG.
  \end{itemize}
The particular choices of parameters summarized here were obtained after a more thorough investigation and optimization among possible values. This parameter exploration is shown in Section \ref{sec: Parameter search} in the supplementary material.

\subsection{Measuring accuracy}
We will compare the performance of several feature extraction methods in the next section. To obtain a more complete perspective, we consider several measurements of accuracy. 

We first use the adjusted Rand index (ARI), which is a widely used cluster validation index for measuring agreement between partitions~\cite{santos2009use}. Given a set $X$ of $n$ points and a partition, e.g., clusterings, of these points,  $P = \lbrace P_{1}, P_{2}, \ldots, P_{r} \rbrace$ into $r$ clusters, the ARI compares it to the ground truth partition $Q = \lbrace Q_{1}, Q_{2}, \ldots, Q_{s} \rbrace$ into $s$ clusters, by calculating
\begin{eqnarray} \label{eq:ARI}
ARI = \frac{\sum_{ij} \binom{n_{ij}}{2} - \left[ \sum_{i} \binom{a_{i}}{2} \sum_{j} \binom{b_{j}}{2} \right]/\binom{n}{2}}{\frac{1}{2} \left[ \sum_{i} \binom{a_{i}}{2} + \sum_{j} \binom{b_{j}}{2} \right] - \left[ \sum_{i} \binom{a_{i}}{2} \sum_{j} \binom{b_{j}}{2} \right] / \binom{n}{2} },
\end{eqnarray}
where $C_{ij} = |P_{i} \cap Q_{j}|$ is the confusion matrix, $a_{i} = |P_{i}|$, and $b_{j} = |Q_{j}|$, for $i = 1,\ldots,r$ and $j = 1,\ldots,s$.  

We next consider the overall accuracy (OA), and the average or weighted accuracy (AA). The overall accuracy is simply the total number of well classified objects w.r.t. the total number of objects, while the average accuracy is the average of the accuracy in each class
\[
OA=\frac{\sum_{i=1}^r |P_i\cap Q_i|}{\sum_{i=1}^r |P_i|}=\frac{\sum_i C_{i,i}}{n},\quad AA=\frac{1}{r}\,\sum_{i=1}^r \frac{|P_i\cap Q_i|}{|P_i|}=\frac{1}{r}\,\sum_{i=1}^r \frac{C_{i,i}}{|a_i|},
\]
with the notations above.

We finally employ the average F-score (FS) and Cohen's kappa coefficient ($\kappa$), which are given by
\[
FS=\frac{1}{r}\,\sum_i \frac{2*C_{i,i}}{2*C_{i,i}+\sum_{j\neq i} (C_{i,j}+C_{j,i})},\quad \kappa=\frac{n\,\sum_i C_{i,i}-\sum_{i,j} C_{i,j}\,C_{j,i}}{n^2-\sum_{i,j} C_{i,j}\,C_{j,i}}.
\]
Both the ARI and $\kappa$ coefficient measure the degree of agreement between clusters.
A strong cluster agreement with the ground truth usually also results in high overall accuracy. 
The average accuracy and the F-score are validation metrics that serve as test scores to ensure that our results are not bias towards a few particular classes. A comparable average accuracy and the F-score across methods is an indication that the algorithms do not favor a few particular classes over the others.
%
\section{Results}\label{sec:results}
We summarize the main results of our numerical experiments on the real hyperspectral images introduced in the previous section. More details are available in the supplementary document.
\subsection{Overall performance}
The following feature extraction algorithms are used in the experiment: principal components analysis~\cite{pearson1901lines} (PCA), Laplacian eigenmaps~\cite{belkin2003laplacian} (LE), diffusion maps~\cite{coifman2006diffusion} (DIF), isomap~\cite{tenenbaum2000global} (ISO), Schroedinger eigenmaps~\cite{cahill2014schroedinger} (SE), transport eigenmaps (TE, including TA and TG). The classification maps for each of the results can be found in the supplement.   

We especially focus on the Adjusted Rand Index, Overall Accuracy, and on the Cohen's kappa coefficient (emphasized in bold in the tables) as the main indicators for the performance of the algorithms.

\subsubsection{Testing on two examples}
We first test TA on the Salinas-B dataset (Table~\ref{tab:salinasB_results}), assuming the class ``lettuce'' is known in SE and TA. Unsurprisingly, the semi-supervised algorithms, SE and TA,  outperform the unsupervised algorithms, PCA, LE, DIF and ISO. The performance of the SE and TA is roughly similar, but with a small but consistent advantage to TA.

\begin{table}[H]
	\begin{center}\scalebox{0.9}{
		\begin{tabular}{ | c | c  c  c  c  c  c |} 
			\hline
			SB			 & PCA     & LE      & DIF     & ISO      & SE   & TA \\
			\hline
			ARI & 0.9429 & 0.9346 & 0.9164 & 0.9440 & 0.9439 & \textbf{0.9463}\\
			OA  & 0.9729 & 0.9685 & 0.9603 & 0.9733 & 0.9762 & \textbf{0.9780}\\
			AA  & 0.9690 & 0.9643 & 0.9564 & 0.9700 & 0.9777 & \textbf{0.9802}\\
			FS  & 0.9693 & 0.9638 & 0.9557 & 0.9696 & 0.9766 & \textbf{0.9795}\\
			$\kappa$ & 0.9682 & 0.9630 & 0.9534 & 0.9687 & 0.9720 & \textbf{0.9742}\\
			\hline
		\end{tabular}}
	\end{center}
	\caption {\footnotesize{Classification results for Salinas-B (SB): assume lettuce (class 14) is known}}
	\label{tab:salinasB_results} 
\end{table}

Classification algorithms frequently mis-classify samples of similar classes due to the similarities in their spectra information. For this reason, we tested the algorithms by grouping similar classes within the Indian Pines and Salinas-B data set to make new ground truths which we denote Indian Pines-G and Salinas-B-G (see Table~\ref{tab:pinesG} and Table~\ref{tab:SalinasBG} in the supplement). 

It turns out SE and TA indeed perform better on grouped Salinas-B (Table~\ref{tab:salinasBG_results}) than on Salinas-B. TA remains to be the best method for the grouped dataset.

\begin{table}[H]
	\begin{center}\scalebox{0.9}{
		\begin{tabular}{ | c | c  c  c  c  c  c |} 
			\hline
			SBG			 & PCA     & LE      & DIF     & ISO    & SE  & TA\\
			\hline
			ARI & 0.9460 & 0.9421 & 0.9154 & 0.9480 & 0.9711 & \textbf{0.9767}\\
			OA  & 0.9791 & 0.9767 & 0.9677 & 0.9795 & 0.9858 & \textbf{0.9880}\\
			AA  & 0.9769 & 0.9750 & 0.9669 & 0.9784 & 0.9819 & \textbf{0.9840}\\
			FS  & 0.9797 & 0.9763 & 0.9697 & 0.9797 & 0.9829 & \textbf{0.9850}\\
			$\kappa$ & 0.9725 & 0.9694 & 0.9576 & 0.9731 & 0.9814 & \textbf{0.9843}\\
			\hline
		\end{tabular}}
	\end{center}
	\caption {\footnotesize{Classification results for Salinas-B-G (SBG): assume lettuce (class 11) is known}}
	\label{tab:salinasBG_results} 
\end{table}

We then test TG on Indian Pines dataset and its grouped version, assuming the class ``soybean'' is known. In this difficult image, the gain of performance in using TG is significant. See Table~\ref{tab:pines_results} and Table~\ref{tab:pinesG_results} below.

\begin{table}[H]
	\begin{center} \scalebox{0.9}{
			\begin{tabular}{ | c | c  c  c  c  c  c | } 
				\hline
				IP		  	 & PCA     & LE      & DIF     & ISO    & SE 	& TG\\
				\hline
				ARI	& 0.4426 & 0.3745 & 0.4210 & 0.3930 & 0.6955 & \textbf{0.7104}\\
				OA  & 0.6761 & 0.6133 & 0.6557 & 0.6309 & 0.7354 & \textbf{0.7431}\\
				AA  & \textbf{0.6403} & 0.5782 & 0.6219 & 0.5979 & 0.6249 & 0.6248\\
				FS  & \textbf{0.6471} & 0.5784 & 0.6212 & 0.5996 & 0.6255 & 0.6250\\
				$\kappa$ & 0.6301 & 0.5592 & 0.6065 & 0.5785 & 0.6982 & \textbf{0.7071}\\
				\hline
		\end{tabular}}
	\end{center}
	\caption {\footnotesize{Classification results for Indian Pines (IP): assume soybean (class 11) is known}.}
	\label{tab:pines_results} 
\end{table}

\begin{table}[H]
	\begin{center}\scalebox{0.9}{
			\begin{tabular}{ | c | c  c  c  c  c  c | } 
				\hline
				IPG		     & PCA     & LE      & DIF     & ISO    & SE   & TG \\
				\hline
				ARI & 0.5330 & 0.4785 & 0.5102 & 0.4902 & 0.8929 & \textbf{0.9264}\\
				OA  & 0.7744 & 0.7307 & 0.7575 & 0.7418 & 0.9088 & \textbf{0.9155}\\
				AA  & 0.6987 & 0.6462 & 0.6883 & 0.6671 & \textbf{0.7111} & 0.7072\\
				FS  & 0.7111 & 0.6479 & 0.6905 & 0.6739 & \textbf{0.7157} & 0.7087\\
				$\kappa$ & 0.6996 & 0.6423 & 0.6770 & 0.6563 & 0.8788 & \textbf{0.8877}\\
				\hline
		\end{tabular}}
	\end{center}
	\caption {\footnotesize{Classification results for Indian Pines-G (IPG): assume soybean (class 10) is known}}
	\label{tab:pinesG_results} 
\end{table}

We remark that ideally the way to implement TE (e.g. TA or TG) should depend on physical interpretation of the data. The above tables show that TA and TG are good for ``arbitrary'' datasets.



\subsubsection{Testing the general TE}
Although being expensive in computation, the use of general TE is needed if information about more than one classes is known. Table~\ref{tab:SB_two_classes}  and Table~\ref{tab:pines_two_classes} show that SE, TA and TG can often perform worse when two classes are known. However, TE gives significant improvements. Here we use $r$ given by \eqref{choice of r} with $small = 0.9$ and $big = 10^4$, and set $a_i= 10$ and $a_i=20$ on the two known classes. 

\begin{table}[H]
	\begin{center} \scalebox{0.9}{
			\begin{tabular}{ | c | c  c c |c | c  c  c  | } 
				\hline
				IP		  	 & SE  & TG & TE & IPG & SE  & TG & TE\\
				\hline
				ARI	& 0.5272 & 0.7693 & \textbf{0.8169} &ARI & 0.4351 & 0.8547 & \textbf{0.9372}\\
				OA  & 0.6855 & 0.8091 & \textbf{0.8268}&OA  & 0.6858 & 0.8967  & \textbf{0.9252} \\
				AA  & 0.6221 & 0.6759 & \textbf{0.6864} &AA  & 0.6431 & 0.7055  & \textbf{0.7221}\\
				FS  & 0.6229 & 0.6766 & \textbf{0.6855} &FS  & 0.6467 & 0.7083& \textbf{0.7242}\\
				$\kappa$ & 0.6409 & 0.7818 & \textbf{0.8024} &$\kappa$ & 0.5821 & 0.8620 &\textbf{0.9004}\\
				\hline
		\end{tabular}}
	\end{center}
	\caption {\footnotesize{Classification results for Indian Pines (IP) and its grouped version (IPG): assume both corn and soybean are known}.}
	\label{tab:pines_two_classes} 
\end{table}

\begin{table}[H]
	\begin{center}\scalebox{0.9}{
			\begin{tabular}{ | c | c  c  c |c | c  c  c |} 
				\hline
				SB			 & SE   & TA  & TE &SBG			 & SE  & TA & TE\\
				\hline
				ARI & 0.9381 & 0.9805 & \textbf{0.9812} &ARI & 0.7916 & 0.9773 &\textbf{0.9823}\\
				OA  & 0.9702 & 0.9909 & \textbf{0.9914} &OA  & 0.9211 & 0.9902 & \textbf{0.9921}\\
				AA  & 0.9671 & 0.9903 & \textbf{0.9908} &AA  & 0.9877 & 0.9877 & \textbf{0.9900}\\
				FS  & 0.9666 & 0.9902 & \textbf{0.9909} &FS  & 0.9365 & 0.9889 & \textbf{0.9906}\\
				$\kappa$ & 0.9651 & 0.9894 & \textbf{0.9899} &$\kappa$ & 0.8966 & 0.9871 & \textbf{0.9896}\\
				\hline
		\end{tabular}}
	\end{center}
	\caption {\footnotesize{Classification results for Salinas-B (SB) and its grouped version (SBG): assume both corn and lettuce are known}}
	\label{tab:SB_two_classes}
\end{table}

In SE, points with positive potential will always be mapped towards the origin. There is no mechanism to handle two different clusters. This explains that SE often perform worse in the above tests. In TA and TG, although the distance from the points to the origin can be modified in different ways by varying $a_i$, points from different classes could still collide after mapping because of their initial locations. The general TE has the power of minimizing the possibility of mixing two known classes since the matrix $r$ provides internal force to group points in the same class.


\subsection{Dependence on the amount of the information}

We performed further experiments on Indian Pines-G and Salinas-G to see how the amount of information available from one particular class affects the performance measures for SE and transport methods TA and TG.



SE and transport methods have very close overall performance on the Indian Pines-G and Salinas-B-G datasets so the comparison may help to understand better the differences between them. As the amount of information increases, so do the performance measures. Figure~\ref{fig:selective_tra_vs_sch} shows the change in performance of SE, TA and TG from using $0 \%$ to using $100 \%$ of the ground truth with increments of $5 \%$ from a particular class.

\begin{figure}[H]
	\centering
	\includegraphics[scale = .18 ]{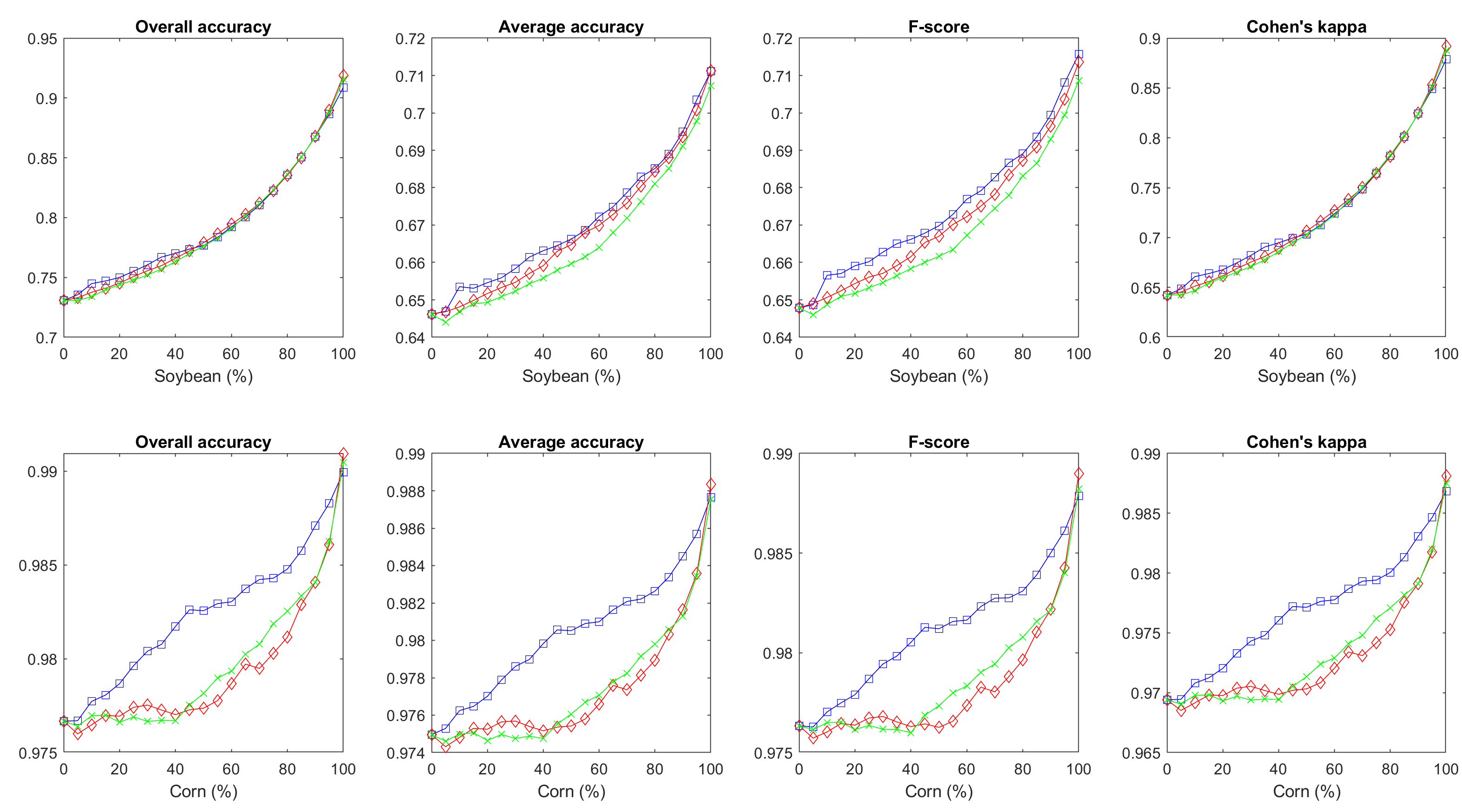}
	\caption{\footnotesize{Classification performance measures for SE (blue squares), TA (red diamonds) and TG (green x's) as a function of the amount of information provided, from $0 \%$ to $100 \%$ with increments of $5 \%$. The Indian Pines-G data set (top row) is used with the advection and potential placed on class 10--soybean. The Salinas-B-G (bottom row) is used with the advection and potential placed on class 10--corn-senesced-green-weeds.}}
	\label{fig:selective_tra_vs_sch}
\end{figure}
Over most of the figure, the SE actually performs slightly better than the TA and TG, with TA and TG only surpassing SE when we have close to $100\%$ of the information on the class. However {\em the difference between the two algorithms remains very small in those two simplified datasets}; this is especially striking on the Indian Pines-G.
%
%

In lieu of these results, it is worth mentioning that in real-life applications, the extra information provided to the algorithms of SE and transport methods does not come directly from the ground truth. Ideally, better and richer cluster information than the ground truth are produced using laboratory measurements and provided to the algorithms as the extra information. These laboratory measurements include various signals representing different materials in a wide range of conditions, e.g., lighting and weather. The use of the ground truth in our aforementioned results is simply due to the unavailability of those better and richer cluster information. It may very well be that with a more complete set of laboratory measurements, we would observe more significant differences between SE and transport methods. 

\subsection{Robustness of Transport eigenmaps}
In a last set of experiments, we investigate the robustness of transport methods TA and TG and some of our other feature extraction algorithms such as PCA, LE, and SE. For this experiment, we have added Gaussian noise to individual data points in the data set before it is processed by the feature extraction algorithms. The added Gaussian noise has a mean of $0$ and we selected $20$ logarithmically spaced values for the standard deviation varying from $10^{0}$ to $10^{5}$ which covers the range for values taken by the individual data points in both set of data. For SE and transport methods, the ground truths for class 10--soybean (Indian Pines-G) and class 11--lettuce (Salinas-B-G) are added to the algorithms. The results are shown on Figure~\ref{fig:performance_noisy_2}. 
\begin{figure}[H]
	\centering
	\includegraphics[scale = .18 ]{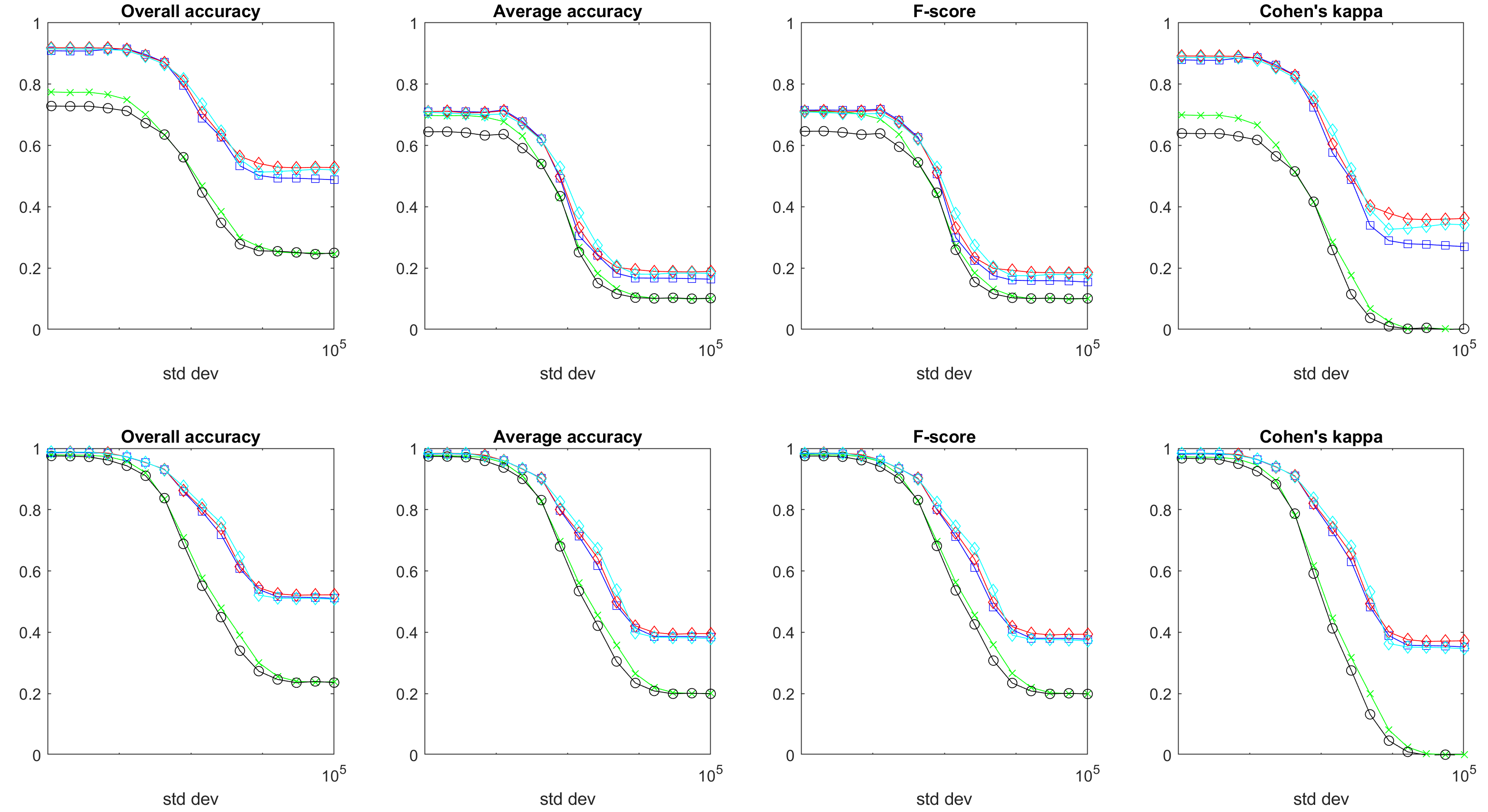}
	\caption{\footnotesize{Classification performance measures for TA (red diamonds), TG (cyan diamonds), SE (blue squares), PCA (green x's), and LE (black circles) as a function of noise. For Indian Pines-G (top row) the potential and advection are placed on class 10--soybean. For Salinas-B-G (bottom row) the potential and advection are placed on class 11--lettuce-romaine.}}
	\label{fig:performance_noisy_2}
\end{figure}
We first gather from the experiments that
SE and transport methods are more resilient to noise than PCA and LE. While the performance of all algorithms naturally decreases very significantly (and interestingly at almost the same mark), SE and transport methods resist better. On the Indian Pines-G, transport methods also end up being the best algorithm by a significant margin, performing $\sim 30\%$ better than SE for large noise whereas they are comparable for small noise. This again suggests that our new Transport algorithm is especially useful in difficult settings where previous methods do not perform well.
%
\section{Conclusion}
In this manuscript, we propose a novel approach to semi-supervised non-linear feature extraction extending the Laplacian eigenmaps. Similar in spirit to previous extension such as Schroedinger eigenmaps, our algorithm is derived from non-linear transport model. We provide a set of experiments based on artificially generated data sets and on publicly available hyperspectral data sets to compare the new method's performance to a variety of algorithms for reducing the dimension of the data provided to a standard classification algorithm.

Those experiments show intriguing possibilities for the new method, which has proved competitive with other algorithms in all settings and significantly outperforms other methods in the more difficult cases (low accuracy because of noise for example). We believe that this demonstrates a strong potential for new methods using advection/gradient flow operators, with in particular the following open questions
\begin{itemize}

\item How to further generalize the transport operator? The choice of the velocity field $v$ in Theorem \ref{main thm} makes the transport operator self-adjoint with respect to an inner product associated with a diagonal matrix $A$. It is natural to investigate the case with a non-diagonal, positive definite $A$.  

\item Can we better relate the choice of an algorithm to the expected structure of the problem? A good example might be time-dependent data, where a clear direction of propagation of the signal would lead to conjecture a even better performance of advection-based eigenmaps.
\item What is the best way to choose the parameters in the general transport method. The intuition provided in Section \ref{sec:controlled_sample} is good only for low dimensional data. When the dimension is high or there are two or more clusters, the choice of $r$ and $a_i$ can be very complicated. We plan to use neural network to attack this problem.
\end{itemize}

\bibliographystyle{siamplain}
\bibliography{references}

\newpage
\appendix
\section{Selecting the field $v$\label{sup:sec:discussion on v}}
\subsection{Limitations of $\bar v_{ij}=a_j-a_i$}
Among all anti-symmetric matrices, this choice is probably the simplest. In this case, the condition
\begin{equation} \label{sup: condition}
(1-\bar v_{ij})x_j=(1+\bar v_{ij})x_i
\end{equation}
 becomes
\begin{equation} \label{sup: condition 1}
(1+a_i-a_j)x_j=(1+a_j-a_i)x_i.
\end{equation}
There are $n$ unknown variables $x_1, \dots, x_n$ and at most $n+1$ independent equations. The solution space is therefore always one-dimensional provided it is non-trivial. For example, suppose every node is connected with the first node, then the value of $x_1$ determines the rest of the variables by \eqref{sup: condition 1}. 

If the nodes are connected as a chain, then the value of any one node in the chain determine the values of the rest.

The only situation one should to be careful is when there is a cycle of nodes. In this case, one needs to check the consistency of the system \eqref{sup: condition 1}. We summarize these observations in the following theorem.

\begin{theorem} \label{thm on cycle}
	The system \eqref{sup: condition 1} has a one-dimensional solution space if for any $l\ge 3$ and any cycle of length $l$, the following holds:
	Without loss of generality and for notation simplicity, assume that the cycle is formed by the first $l$ nodes $1, 2, \dots, l$. Then
	\begin{equation} \label{cycle condition}
	\begin{aligned}
	&(1+a_1-a_2)(1+a_2-a_3)\cdots (1+a_{l-1}-a_l)(1+a_l-a_1)\\
	&=(1+a_2-a_1)(1+a_3-a_2)\cdots(1+a_l-a_{l-1})(1+a_1-a_l)   
	\end{aligned}    
	\end{equation}
\end{theorem}
The following corollaries provide more easy-to-check condition than \eqref{cycle condition}.
\begin{corollary}
	If the graph is a tree, i.e., without cycles, then the system \eqref{sup: condition 1} has solutions.
\end{corollary}

\begin{corollary} \label{coro about 2}
	If the cardinality of the set $\{a_i: i=1,2,\dots,n\}$ is $2$, then the system \eqref{sup: condition 1} has solutions.
\end{corollary}
\begin{proof}
	We will show by induction that \eqref{cycle condition} always holds when $a_i$'s attain only two values. When $l=3$, for any indices $i, j, k$, 
	\begin{equation} \label{l=3}
	(1+a_i-a_j)(1+a_j-a_k)(1+a_k-a_i)=(1+a_j-a_i)(1+a_k-a_j)(1+a_i-a_k)    
	\end{equation}
	holds whenever $a_i, a_j$ and $a_k$ attain two values. This proves the base case. 
	
	Assume \eqref{cycle condition} is valid for $l=n$. Let $l=n+1$. We need to show that
	\begin{equation}\label{l=n+1}
	\begin{aligned}
	&(1+a_1-a_2)(1+a_2-a_3)\cdots (1+a_{n-1}-a_n)(1+a_n-a_{n+1})(1+a_{n+1}-a_1)\\
	&=(1+a_2-a_1)(1+a_3-a_2)\cdots(1+a_n-a_{n-1})(1+a_{n+1}-a_n)(1+a_1-a_{n+1}).   
	\end{aligned}       
	\end{equation}
	Write LHS of the above equality as
	\begin{align*}
	&\text{LHS of }\eqref{l=n+1}\cdot \frac{1+a_n-a_1}{1+a_n-a_1}\\
	&=(1+a_1-a_2)(1+a_2-a_3)\cdots (1+a_{n-1}-a_n)(1+a_n-a_1)\frac{(1+a_n-a_{n+1})(1+a_{n+1}-a_1)}{1+a_n-a_1} \\
	&=(1+a_2-a_1)(1+a_3-a_2)\cdots(1+a_n-a_{n-1})(1+a_1-a_n)\frac{(1+a_n-a_{n+1})(1+a_{n+1}-a_1)}{1+a_n-a_1},
	\end{align*}
	where we applied the induction hypothesis for $l=n$ in the last step.
	
	It is clear that \eqref{l=n+1} will follow once we show 
	$$
	(1+a_1-a_n)\frac{(1+a_n-a_{n+1})(1+a_{n+1}-a_1)}{1+a_n-a_1}=(1+a_{n+1}-a_n)(1+a_1-a_{n+1}).
	$$
	This was verified in \eqref{l=3} as the base case $l=3$.
	
\end{proof}

The following example shows that Corollary \ref{coro about 2} is sharp in the sense that $2$ cannot be replaced with $3$ or more.

Suppose the first three nodes form a cycle. Then \eqref{cycle condition} becomes
\begin{equation} \label{eq: 3-cycle}
(1+a_1-a_2)(1+a_2-a_3)(1+a_3-a_1)=(1+a_2-a_1)(1+a_3-a_2)(1+a_1-a_3),
\end{equation}
which can be simplified to
\begin{equation} \label{eq: 3-cycle simple}
(a_1-a_2)(a_2-a_3)(a_3-a_1)=0.
\end{equation}
A quick way to see the equivalence of the above two equations is to note that the difference of the two sides of \eqref{eq: 3-cycle} is a polynomial that vanishes at $a_1=a_2$, $a_1=a_3$ and $a_2=a_3$, and thus has factors $a_1-a_2$, $a_1-a_3$ and $a_2-a_3$. Now it is clear that equation \eqref{eq: 3-cycle simple} cannot hold if $a_1, a_2, a_3$ are all different.

\subsection{The choice of $\bar v_{ij}=\frac{a_j-a_i}{a_j+a_i} r_{ij}$}
We argue that $\bar v_{ij}=\frac{a_j-a_i}{a_j+a_i} r_{ij}$ is a very natural condition for the self-adjoint transport operator. 
\subsubsection{Relation to the minimization problem.} Just like the Laplacian operator is closely related with the problem of minimizing $\sum_{i,j}(y_i-y_j)^2w_{ij}$, transport operator, whose major term is the Laplacian, corresponds to minimizing a more general quadratic form.

Let $Q$ be a symmetric matrix. Write 
\begin{equation} \label{general quadratic form}
E(\tb y):=\tb y^t Q\tb y=\sum_{ij}(\al_iy_i^2+\al_jy_j^2-2\be_{ij}y_iy_j)w_{ij},    
\end{equation}
where $\be$ is symmetric. When $Q$ is the Laplacian, $\al_i\equiv 1$ and $\beta_{ij}\equiv 1$. 

We want to minimize $E(\tb y)$ under a general constraint $\tb y^tX\tb y=1$, where $X$ is a positive definite matrix. By Lagrangian multiplier method, this problem reduces to the general eigenvalue problem $Q\tb y=\lambda X\tb y$. One can expect that the eigenvalues of $X^{-1}Q$ are real, and to relate it with $T_v^r$, simply let $X^{-1}Q=T_v^r$, i.e., $Q=XT_v^r$. It is straightfoward to see that
$$
Q=2\begin{bmatrix}
d_1\al_1 & &\\
& \ddots &\\
& & d_n\al_n
\end{bmatrix}
-2(w_{ij}\be_{ij}),
$$
where $d_i=\sum_j w_{ij}$, and
$$
T_v^r=\begin{bmatrix}
D_1 & &\\
& \ddots &\\
& & D_n
\end{bmatrix}-(w_{ij}(r_{ij}+\bar v_{ij})),
$$
where $D_i:=\sum_j w_{ij}(r_{ij}-\bar v_{ij})$. So $Q=XT_v^r$ is equivalent to
\begin{equation}
2\begin{bmatrix}
d_1\al_1 & &\\
& \ddots &\\
& & d_n\al_n
\end{bmatrix}-2(w_{ij}\be_{ij})=X\begin{bmatrix}
D_1 & &\\
& \ddots &\\
& & D_n
\end{bmatrix}-X(w_{ij}(r_{ij}+\bar v_{ij})).
\end{equation}
One may try to find a special solution by letting
\begin{equation} \label{eq 1}
2\begin{bmatrix}
d_1\al_1 & &\\
& \ddots &\\
& & d_n\al_n
\end{bmatrix}=X\begin{bmatrix}
D_1 & &\\
& \ddots &\\
& & D_n
\end{bmatrix}  
\end{equation}
and 
\begin{equation} \label{eq 2}
2(w_{ij}\be_{ij})=X(w_{ij}(r_{ij}+\bar v_{ij}))    
\end{equation}

From \eqref{eq 1} one can deduce that $X$ must be diagonal. Write $X=\text{diag}(x_i)$. Then
\begin{equation} \label{eq 3}
2d_i\al_i=x_i\sum_jw_{ij}(r_{ij}-\bar v_{ij}).
\end{equation}
Equation \eqref{eq 2} implies $2w_{ij}\be_{ij}=x_iw_{ij}(r_{ij}-\bar v_{ij})$ and thus
\begin{equation} \label{r+v}
r_{ij}+\bar v_{ij}=\frac{2\be_{ij}}{x_i}.
\end{equation}
Since $\bar v_{ij}=-\bar v_{ij}$ and $r_{ij}=r_{ji}$, the above equation yields
$$\bar v_{ij}=\frac{2\beta_{ij}}{x_i}-r_{ij}=-\bar v_{ij}=-\frac{\beta_{ij}}{x_j}+r_{ij},$$
which is the same as
\begin{equation} \label{r-v}
r_{ij}-\bar v_{ij}=\frac{\beta_{ij}}{x_j}.
\end{equation}
Combine \eqref{r+v} and \eqref{r-v}, and we get
\begin{align}
&r_{ij}=\be_{ij}\left(\frac{1}{x_i}+\frac{1}{x_j} \right)\\
&\bar v_{ij}=\be_{ij}\left(\frac{1}{x_i}-\frac{1}{x_j} \right).
\end{align}
Clearly, 
\begin{equation*}
\frac{\bar v_{ij}}{r_{ij}}=\frac{x_j-x_i}{x_j+x_i},
\end{equation*}
which is of the right form.

Finally, invoking \eqref{eq 3},
\begin{equation*}
\al_i=\frac{x_i}{2d_i}\sum_jw_{ij}(r_{ij}-\bar v_{ij})=\frac{x_i}{d_i}\sum_j \frac{w_{ij}\be_{ij}}{x_j}.
\end{equation*}
In conclusion, for given $x_i$ and $\be_{ij}$, one can find the right $\al_i$ and $r_{ij}, \bar v_{ij}$ so that the minimization problem of the general quadratic form \eqref{general quadratic form} corresponds to a general form of transport operator with the property $\bar v_{ij}=\frac{a_j-a_i}{a_j+a_i} r_{ij}$. 

\subsubsection{Continuous analogues} 
The classic transport operator in Euclidean space is
\begin{equation}
F(y)=\Delta y-\text{div}(\tb vy).
\end{equation}
As we used $-\Delta$ instead of $\Delta$ as the (discrete) graph Laplacian, we should modify the above definition and regard the following operator as the continuous analogue of $T$:
\begin{equation*}
F(y)=\Delta y+\text{div}(\tb v y),
\end{equation*}
where $y$ is a real-valued function and $\tb v$ is vector-valued. 

We aim to find $\tb v$ such that $F(y)$ is self-adjoint with respect to some inner product $\langle f,g\rangle_x:=\int f(t)g(t)x(t)\, dt$ associated with the function $x$. More precisely, we need to show that there exists a suitable function $x$  such that for any real-valued test functions $y$ and $z$,
\begin{equation}
\langle F(y),z\rangle_x=\langle y, F(z)\rangle_x.
\end{equation}
By partial summation and straightforward calculations, 
\begin{align*}
\langle F(y),z\rangle&=\int F(y)(t) z(t) x(t)\,dt=\int (\Delta y+\text{div}(\tb vy))zx\\
&=\int y\Delta(zx)-\int \tb v y\cdot\nabla(zx)=\int y[\Delta(zx)-\tb v\cdot \nabla(zx)]\\
&=\int y(x\Delta z+z\Delta x+2\nabla z\cdot \nabla x-x\tb v\cdot \nabla z-z\tb v\cdot \nabla x)\\
=\int y(x\Delta z+x\nabla z\cdot &\tb v+xz\text{div} \tb v)+\int y(z\Delta x+2\nabla z\cdot \nabla x-2x\tb v\cdot\nabla z-xz\text{div}\tb v-z\tb v\cdot\nabla x)\\
=\int y(\Delta z+\text{div}& (\tb vz))x+\int y(z\Delta x+2\nabla z\cdot \nabla x-2x\tb v\cdot\nabla z-xz\text{div}\tb v-z\tb v\cdot\nabla x)\\
=\langle y, F(z)\rangle&+\int y(z\Delta x+2\nabla z\cdot \nabla x-2x\tb v\cdot\nabla z-xz\text{div}\tb v-z\tb v\cdot\nabla x).
\end{align*}
Therefore, we need to set 
$$
\int y(z\Delta x+2\nabla z\cdot \nabla x-2x\tb v\cdot\nabla z-xz\text{div}\tb v-z\tb v\cdot\nabla x)=0.
$$
As $y$ is arbitrary, we have
$$
\begin{cases}
z(\Delta x-x\text{div}\tb v-\tb v\cdot \nabla x)=0\\
2\nabla z\cdot (\nabla x-x\tb v)=0.
\end{cases}
$$
Since $z$ is arbitrary as well, 
$$
\begin{cases}
\Delta x-x\text{div}\tb v-\tb v\cdot \nabla x=0\\
\nabla x-x\tb v=0.  
\end{cases}
$$
It is easy to see that the first equation is the divergence of the second one. Thus the only equation we need is
\begin{equation} \label{ctns eq}
\nabla x=x\tb v
\end{equation}
We claim that this equation agrees with that in the discrete setting. For the discrete operator $T$, if we set $\bar v_{ij}=a_j-a_i$, which means that $v_{ij}=2\bar v_{ij}w_{ij}=2(a_j-a_i)w_{ij}=2(\nabla a)_{ij}$ using the  rules of translation, then this corresponds to take $\tb v=\nabla 2a$ for some real-valued function $a$. As
$$
\nabla x``="(x_j-x_i)w_{ij},
$$
and
$$
x\nabla 2a``="\frac{x_j+x_i}{2}2(a_j-a_i)w_{ij}=(x_j+x_i)(a_j-a_i)w_{ij}
$$
the equation \eqref{ctns eq} in Euclidean space becomes
$$
(x_j-x_i)w_{ij}=(x_j+x_i)(a_j-a_i)w_{ij},
$$
which is exactly \eqref{sup: condition 1}, the equation in the discrete setting.

Note that in the continuous setting, \eqref{ctns eq} can be solved easily
\begin{align*}
&\nabla x=x\nabla 2a\\
&\nabla \log x=\nabla 2a\\
&x=ce^{2a}.
\end{align*}
The above process is invalid in the discrete setting and thus we have to impose condition \eqref{cycle condition} in Theorem \ref{thm on cycle}. However, setting $\bar v_{ij}=\frac{a_j-a_i}{a_j+a_i}$ gives $$v_{ij}=2\bar v_{ij}w_{ij}=\frac{a_j-a_i}{\frac{a_j+a_i}{2}}w_{ij}=\left(\frac{\nabla a}{a} \right)_{ij},$$
which suggest taking $\tb v=\frac{\nabla a}{a}$ in \eqref{ctns eq}. In this situation, we have
$$
\nabla x=x\frac{\nabla a}{a},
$$
which always has solutions $x=ca$ in both continuous and discrete settings. We conclude that $\bar v_{ij}=\frac{a_j-a_i}{a_j+a_i}$ is a natural flow field that makes $T$ self-adjoint.

\section{Hyperspectral dataset\label{sup:sec:hyperspectral data}}
We present here in more details the structure of our dataset, with first the ground truth and an example of spectral band for the Indian Pines data set.

\begin{figure}[H]
	\centering
	\includegraphics[scale = .5 ]{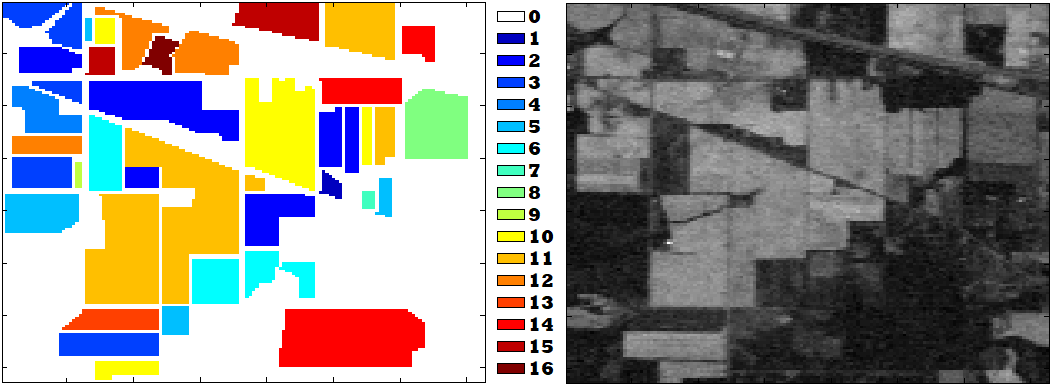}
	\caption{\footnotesize{Ground truth of Indian Pines data set (left) and sample band of Indian Pines data set (right.)}}
	\label{fig:Indian_pines}
\end{figure}

The ground truth classes for the Indian Pines data set are listed below, together with the number of samples in each class.

\begin{table}[H]
	\begin{center}
		\begin{tabular}{ | l | l | l | } 
			\hline
			\# & Class & Sample \\
			\hline \hline
			0 &	Empty-space                   &  10776\\ \hline
			1 &	Alfalfa                       &  46   \\ \hline
			2 &	Corn-notill                   &  1428 \\ \hline
			3 & Corn-mintill                  &  830  \\ \hline
			4 & Corn                          &  237  \\ \hline
			5 & Grass-pasture                 &  483  \\ \hline
			6 & Grass-trees                   &  730  \\ \hline
			7 & Grass-pasture-mowed           &  28   \\ \hline
			8 & Hay-windrowed                 &  478  \\ \hline
			9 & Oats                          &  20   \\ \hline
			10&	Soybean-notill                &  972  \\ \hline
			11&	Soybean-mintill               &  2455 \\ \hline
			12&	Soybean-clean                 &  593  \\ \hline
			13&	Wheat                         &  205  \\ \hline
			14&	Woods                         &  1265 \\ \hline
			15&	Buildings-Grass-Trees-Drives  &  386  \\ \hline
			16&	Stone-Steel-Towers            &  93   \\ \hline
		\end{tabular}
	\end{center}
	\caption {\footnotesize{Indian Pines classes.}}
	\label{tab:pines} 
\end{table}

The Salinas dataset is very similar to the Indian Pines with
\begin{figure}[H]
	\centering
	\includegraphics[scale = .5 ]{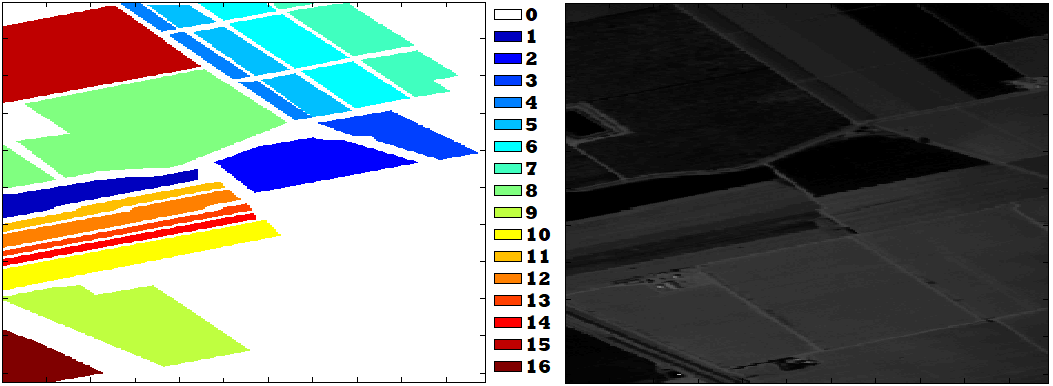}
	\caption{\footnotesize{Ground truth of Salinas data set (left) and sample band of Salinas data set (right.)}}
	\label{fig:Salinas}
\end{figure}
And the list of classes comparable as well
\begin{table}[H]
	\begin{center}
		\begin{tabular}{ | l | l | l | } 
			\hline
			\# & Class & Sample \\
			\hline \hline
			0 &	Empty-space                   & 10776\\ \hline
			1 &	Broccoli-green-weeds-1		  & 2009\\ \hline
			2 &	Broccoli-green-weeds-2        &	3726\\ \hline
			3 &	Fallow	                      & 1976\\ \hline
			4 &	Fallow-rough-plow			  &	1394\\ \hline
			5 &	Fallow-smooth				  &	2678\\ \hline
			6 &	Stubble						  &	3959\\ \hline
			7 &	Celery						  &	3579\\ \hline
			8 &	Grapes-untrained			  &	11271\\ \hline
			9 &	Soil-vineyard-develop		  &	6203\\ \hline
			10&	Corn-senesced-green-weeds	  &	3278\\ \hline
			11&	Lettuce-romaine-4wk			  &	1068\\ \hline
			12&	Lettuce-romaine-5wk			  &	1927\\ \hline
			13&	Lettuce-romaine-6wk			  &	916\\ \hline
			14&	Lettuce-romaine-7wk			  &	1070\\ \hline
			15&	Vineyard-untrained			  &	7268\\ \hline
			16&	Vineyard-vertical-trellis	  &	1807\\ \hline
		\end{tabular}
	\end{center}
	\caption {\footnotesize{Salinas classes.}}
	\label{tab:salinas} 
\end{table}
%
%
%
%

We introduce the reduced dataset Salinas-B, which allows for faster calculations than the full Salinas dataset and proved useful in exploring the space of parameters for optimization. 

\begin{figure}[H]
	\centering
	\includegraphics[scale = .5 ]{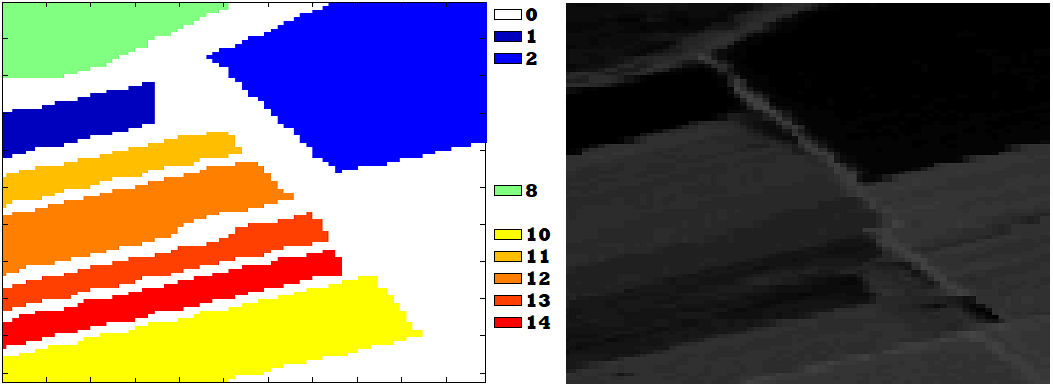}
	\caption{\footnotesize{Ground truth of Salinas-B data set (left) and sample band of Salinas-B data set (right.)}}
	\label{fig:SalinasB}
\end{figure}

\begin{table}[H]
	\begin{center}
		\begin{tabular}{ | l | l | l | } 
			\hline
			\# & Class & Sample \\
			\hline \hline
			0 &	Empty-space                   & 5826\\ \hline
			1 &	Broccoli-green-weeds-1		  & 914\\ \hline
			2 &	Broccoli-green-weeds-2 		  &	1854\\ \hline
			8 &	Grapes-untrained			  &	1240\\ \hline
			10&	Corn-senesced-green-weeds	  &	1959\\ \hline
			11&	Lettuce-romaine-4wk			  &	655\\ \hline
			12&	Lettuce-romaine-5wk			  &	1229\\ \hline
			13&	Lettuce-romaine-6wk			  &	616\\ \hline
			14&	Lettuce-romaine-7wk			  &	707\\ \hline
		\end{tabular}
	\end{center}
	\caption {\footnotesize{Salinas-B classes.}}
	\label{tab:SalinasB} 
\end{table}

Finally we also use variant classifications on the Indian Pines and Salinas datasets, denoted Indian Pines-G and Salinas-B-G, which group together similar classes.

\begin{table}[H]
	\begin{center}
		\begin{tabular}{ | l | l | l | } 
			\hline
			\# & Class & Sample \\
			\hline \hline
			0 &	Empty-space                   &  10776\\ \hline
			1 &	Alfalfa                       &  46   \\ \hline
			2 &	Corn	                      &  2495 \\ \hline
			5 & Grass		                  &  1241 \\ \hline
			8 & Hay-windrowed                 &  478  \\ \hline
			9 & Oats                          &  20   \\ \hline
			10&	Soybean		                  &  4020 \\ \hline
			13&	Wheat                         &  205  \\ \hline
			14&	Woods                         &  1265 \\ \hline
			15&	Buildings-Grass-Trees-Drives  &  386  \\ \hline
			16&	Stone-Steel-Towers            &  93   \\ \hline
		\end{tabular}
	\end{center}
	\caption {\footnotesize{Indian Pines-G classes, ground truth with corresponding grouped labels.}}
	\label{tab:pinesG} 
\end{table}

\begin{table}[H]
	\begin{center}
		\begin{tabular}{ | l | l | l | } 
			\hline
			\# & Class & Sample \\
			\hline \hline
			0 &	Empty-space                   & 5826\\ \hline
			1 &	Broccoli-green-weeds-1		  & 914\\ \hline
			2 &	Broccoli-green-weeds-2 		  &	1854\\ \hline
			8 &	Grapes-untrained			  &	1240\\ \hline
			10&	Corn-senesced-green-weeds	  &	1959\\ \hline
			11&	Lettuce-romaine			  	  &	3207\\ \hline
		\end{tabular}
	\end{center}
	\caption {\footnotesize{Salinas-B-G classes, ground truth with corresponding grouped labels.}}
	\label{tab:SalinasBG} 
\end{table}

\section{Parameter exploration and search} ~\label{sec: Parameter search}
We present more in details in this section how we chose the particular values of the parameters for our algorithms.
The choice came from separate optimizations, varying one parameter while keeping others fixed. We refer to the main article for the description of where the parameters enter in the various algorithms.

\subsection{Selecting $k$, $\sigma$ and $m$}
The first series of tests were done to determine $k$, the number of nearest neighbors used to construct the graph, $\sigma$ the deviation used to construct the weight matrix and $m$ the number of generalized eigenvectors (or intrinsic dimension) used in the mapping. 

We investigate $\sigma$ in the last row of Figure~\ref{fig:parameter_k_and_no_dim} by looking at various measures of performance of the Laplacian eigenmap for $k=12$ and $m=50$ (Indian Pines) or $m=25$ (Salinas-B). Similarly we look at $k$ in the middle row of Figure~\ref{fig:parameter_k_and_no_dim} by choosing $\sigma=1$ and the same values of $m$. The first row of Figure~\ref{fig:parameter_k_and_no_dim} analyzes the optimal choice of $m$ by fixing $k=12$,$\sigma=1$.
\begin{figure}[H]
	\centering
	\includegraphics[scale = .7 ]{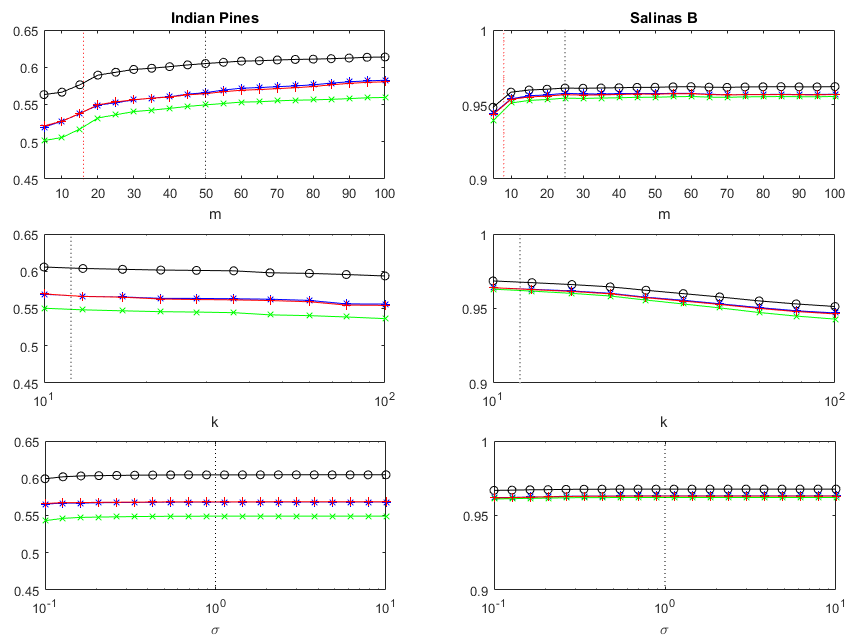}
	\caption{\footnotesize{Optimal analysis of the parameters $m$ (top), $k$ (middle), and $\sigma$ (bottom): overall accuracy (black circles), average accuracy (blue stars), F-score (red pluses), and Cohen's kappa coefficient (green x's). The red dashed vertical lines represent the number of classes and the black dashed vertical lines represent our choice for the corresponding parameter.}}
	\label{fig:parameter_k_and_no_dim}
\end{figure}
The results from the optimizations show that the choices of $k$ and $\sigma$ do not greatly affect the results. We do observe that, since higher values for $k$ introduce more connections between distinct classes, this leads to a higher number of mis-classified samples and explains the slight decline in performance observed in the middle row of Figure~\ref{fig:parameter_k_and_no_dim}. Given those results,  we fixed the value for the weight parameter at $\sigma = 1$ for simplicity, while the value for the number of neighbors is fixed at $k = 12$ to ensure that we have connected graphs but that $k$ is not too large.

The Indian Pines data set had already been investigated and the value $m=50$ for the intrinsic dimension reflects both what was seen previously in literature and our results in Figure~\ref{fig:parameter_k_and_no_dim} (top row). Since the Salinas-B data set has never been analyzed before, the intrinsic dimension $m=25$ in that case was chosen solely based on our investigation in Figure~\ref{fig:parameter_k_and_no_dim} (top row).

\subsection{Selecting $\alpha$ and $\beta$ for the Schroedinger and Transport eigenmaps}
The choices of $k$, $\sigma$ and $m$ are enough for all algorithms except the Schroedinger and Transport eigenmaps. The potential $V$ and function $\mu$ in those cases are naturally determined by the a priori knowledge, leaving only $\alpha$ which determines the strength of the potential in the Schroedinger eigenmap, and $\beta$ which controls the strength of the advection in the Transport eigenmap.

As described in the main article, for the Schroedinger eigenmaps, we introduced the parameter $\hat{\alpha}$ defined by $\alpha = \hat{\alpha} \cdot \tr(L) / \tr(V)$ to compare the impact of the Laplacian matrix and the potential in the algorithm. The result of the search on $\hat\alpha$ is shown on Figure~\ref{fig:parameter_alphahat_sch}.
\begin{figure}[H]
	\centering
	\includegraphics[scale = .17 ]{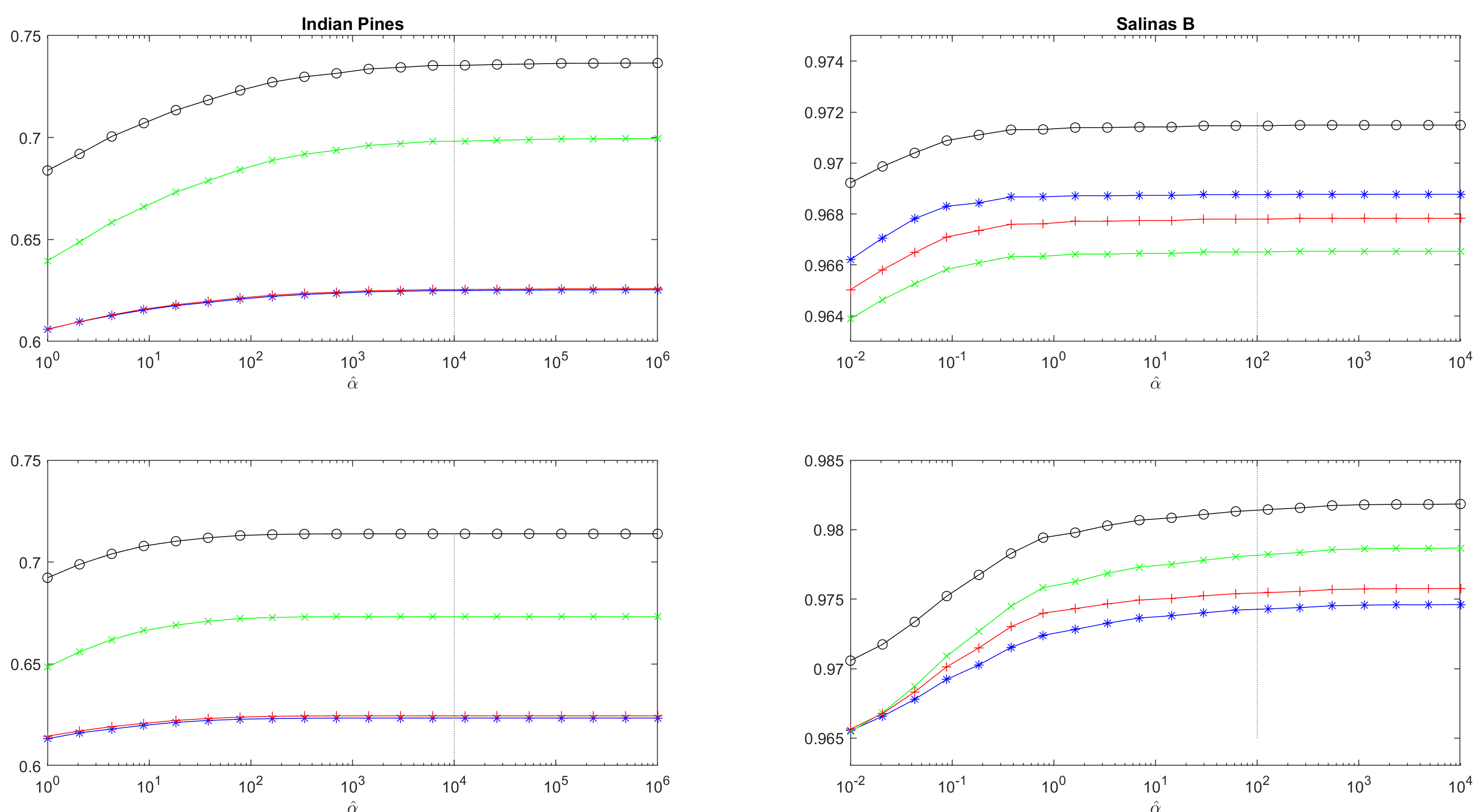}
	\caption{\footnotesize{Optimal analysis of the parameter $\hat{\alpha}$. For the Indian Pines data set (left), the potential is placed on class 11--soybean (top) then on class 2--corn (bottom). For the Salinas-B data set (right), the potential is placed on class 11--lettuce (top) then on class 10--corn (bottom). The following performance measures are reported: overall accuracy (black circles), average accuracy (blue stars), F-score (red pluses), and Cohen's kappa coefficient (green x's). The black dashed vertical lines represent our choice for the parameter $\hat{\alpha}$.}}
	\label{fig:parameter_alphahat_sch}
\end{figure}
 For transport by advection (TA), we searched for the optimal value of $\beta$, again for both Indian Pines and Salinas-B,
 on Figure~\ref{fig:parameter_alpha_tra}.
\begin{figure}[H]
	\centering
	\includegraphics[scale = .17 ]{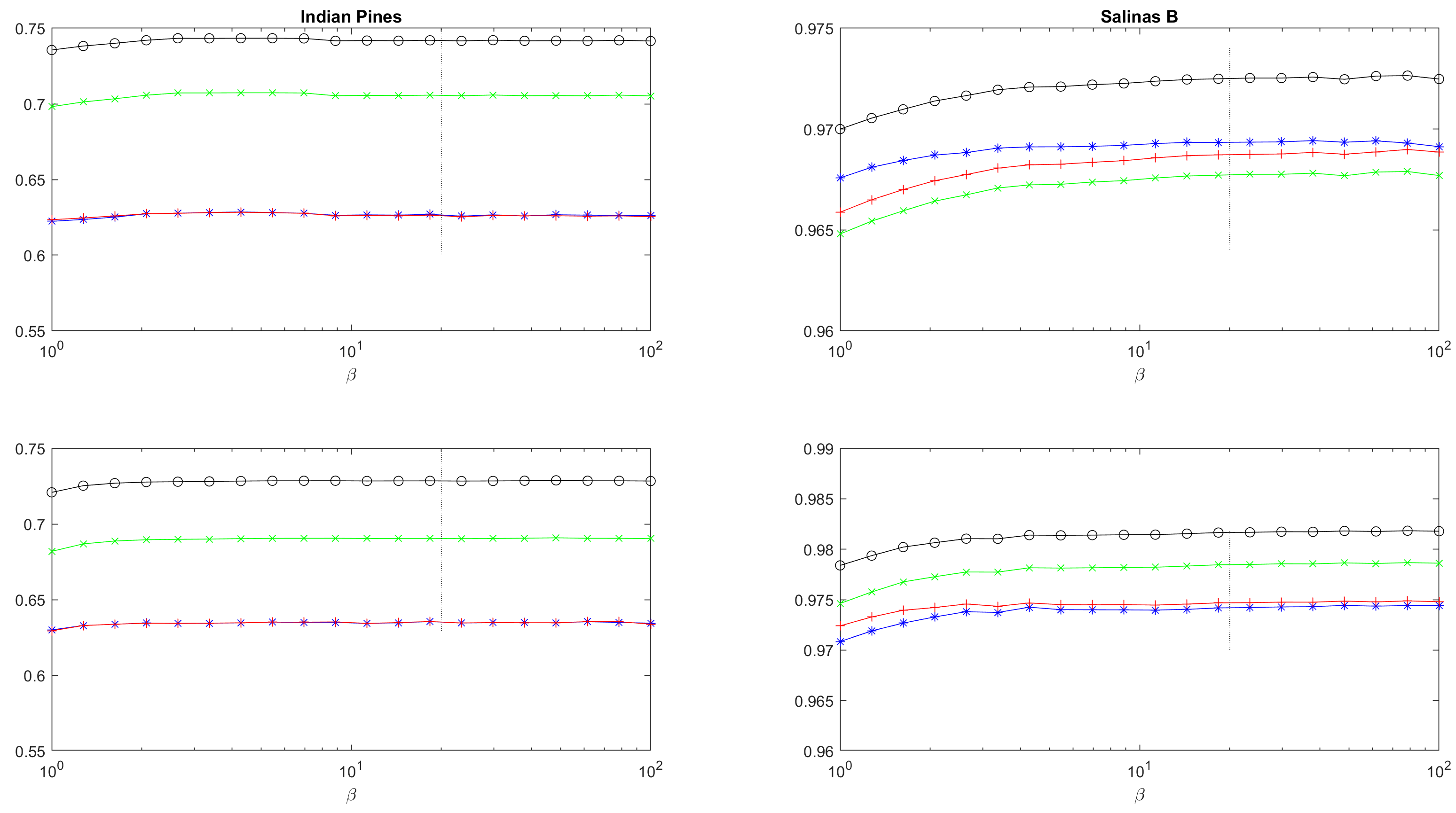}
	\caption{\footnotesize{Optimal analysis of the parameter $\beta$. For the Indian Pines data set (left), the advection is placed on class 11--soybean (top) then on class 2--corn (bottom). For the Salinas-B data set (right), the advection is placed on class 11--lettuce (top) then on class 10--corn (bottom). The following performance measures are reported: overall accuracy (black circles), average accuracy (blue stars), F-score (red pluses), and Cohen's kappa coefficient (green x's). The black dashed vertical lines represent our choice for the parameter $\beta$.}}
	\label{fig:parameter_alpha_tra}
\end{figure}

The same is done for transport gradient flow (TG).

\begin{figure}[H]
	\centering
	\includegraphics[scale = .17 ]{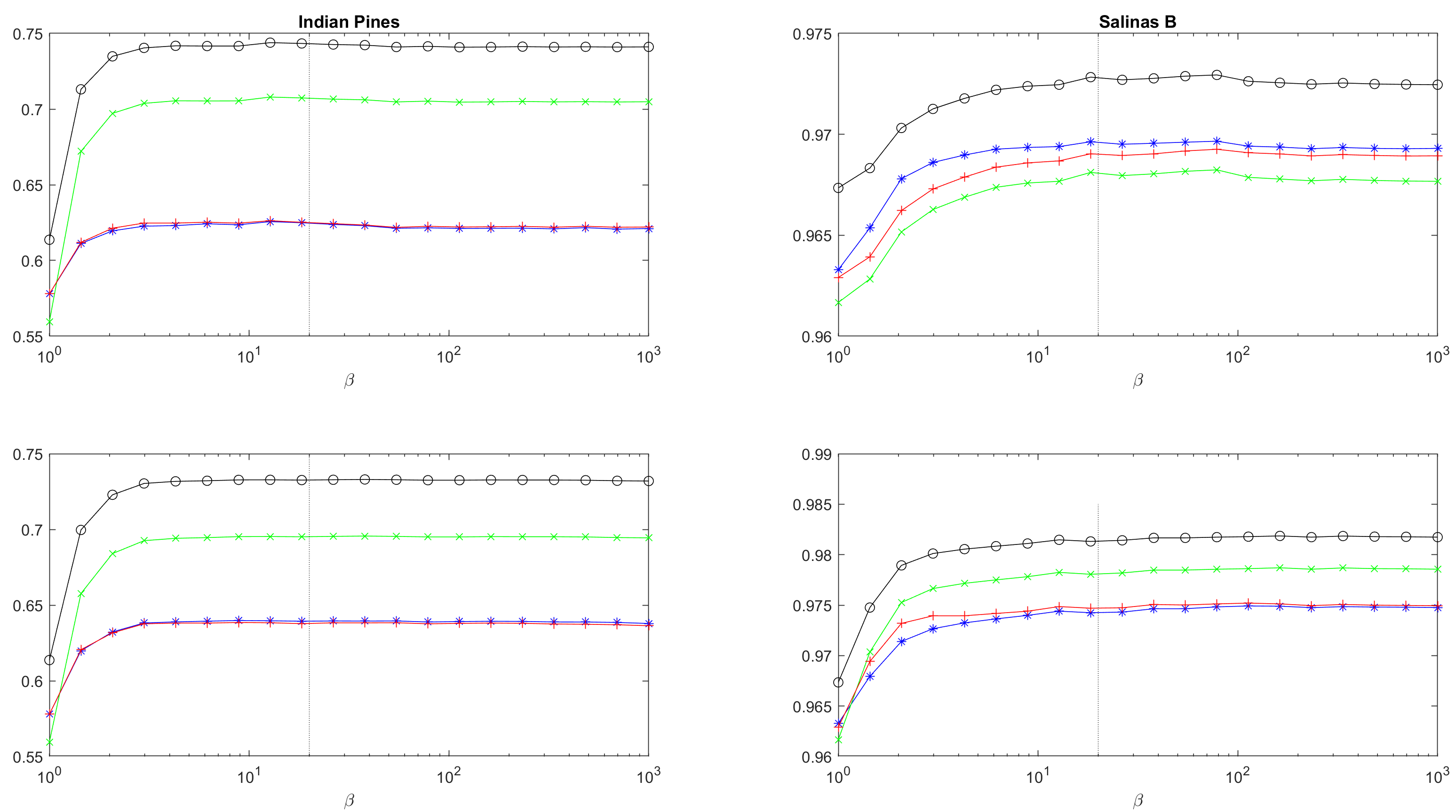}
	\caption{\footnotesize{Optimal analysis of the parameter $\beta$. For the Indian Pines data set (left), the known class is class 11--soybean (top) and then class 2--corn (bottom). For the Salinas-B data set (right), the known class is class 11--lettuce (top) and then class 10--corn (bottom). The following performance measures are reported: overall accuracy (black circles), average accuracy (blue stars), F-score (red pluses), and Cohen's kappa coefficient (green x's). The black dashed vertical lines represent our choice for the parameter $\beta$.}}
	\label{fig:parameter_alpha_t2}
\end{figure}

The investigation for both $\hat\alpha$ and $\beta$ show that the performance is not much impacted by the choice, provided that both are chosen large enough: Too small values of $\hat\alpha$ or $\beta$ do not allow to take advantage of the available prior information. 
Based on those results, we take $\hat{\alpha} = 10^{4}$ for the Indian Pines data set and $\hat{\alpha} = 10^{2}$ for the Salinas-B data set for SE, and using $\beta = 20$ for both the Indian Pines and the Salinas-B data set for TA and TG.
%
\subsection{Results: Classification maps\label{sec:results maps}}
We show here some graphical representations of the classification maps, on the Indian Pines dataset.
\begin{figure}[H]
	\centering
	\includegraphics[scale = .19 ]{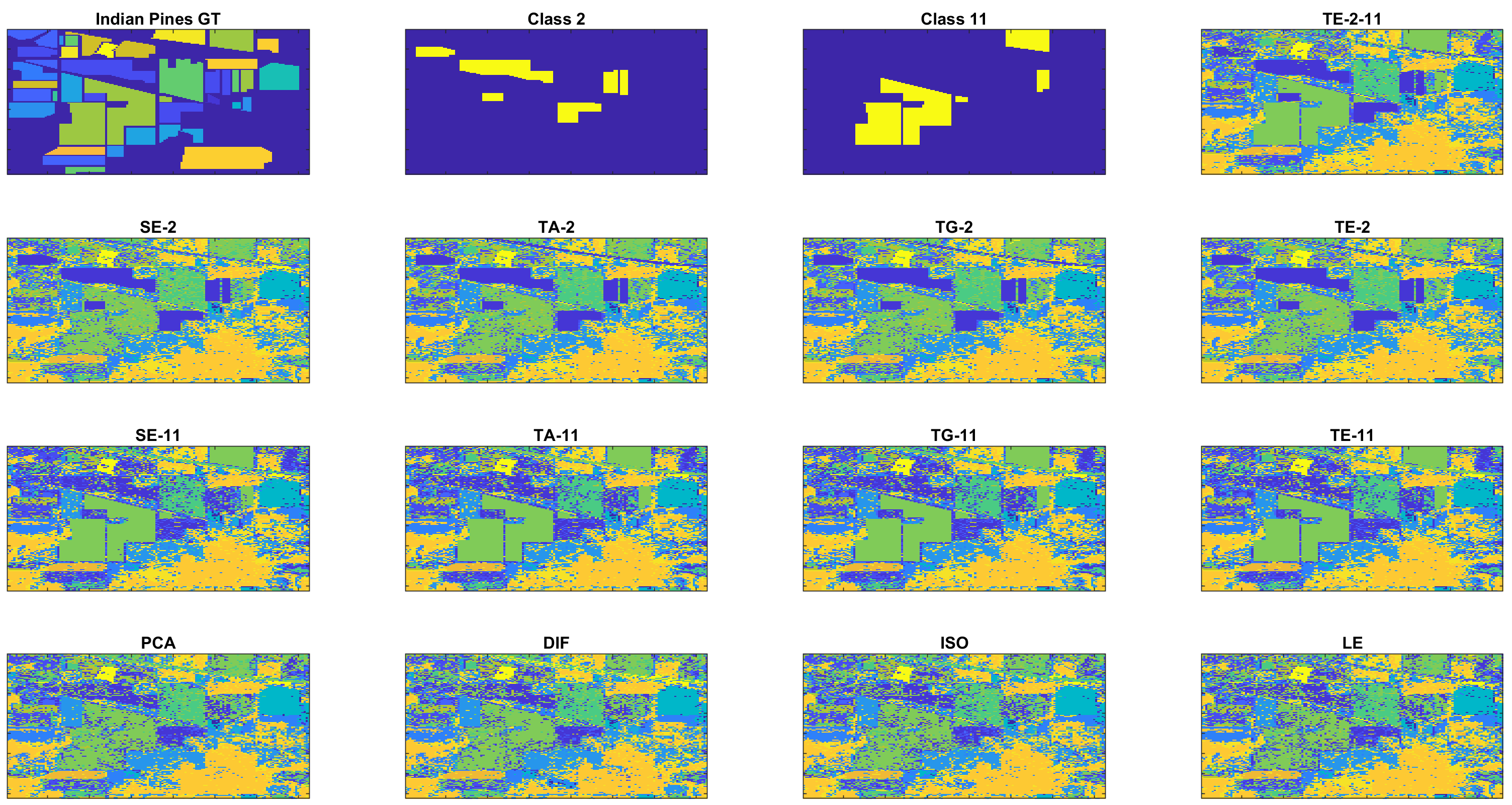}
	\caption{\footnotesize{Classification map of Indian Pines}}
	\label{figPines_map}
\end{figure}
Then on on the Salinas B dataset.
\begin{figure}[H]
	\centering
	\includegraphics[scale = .19 ]{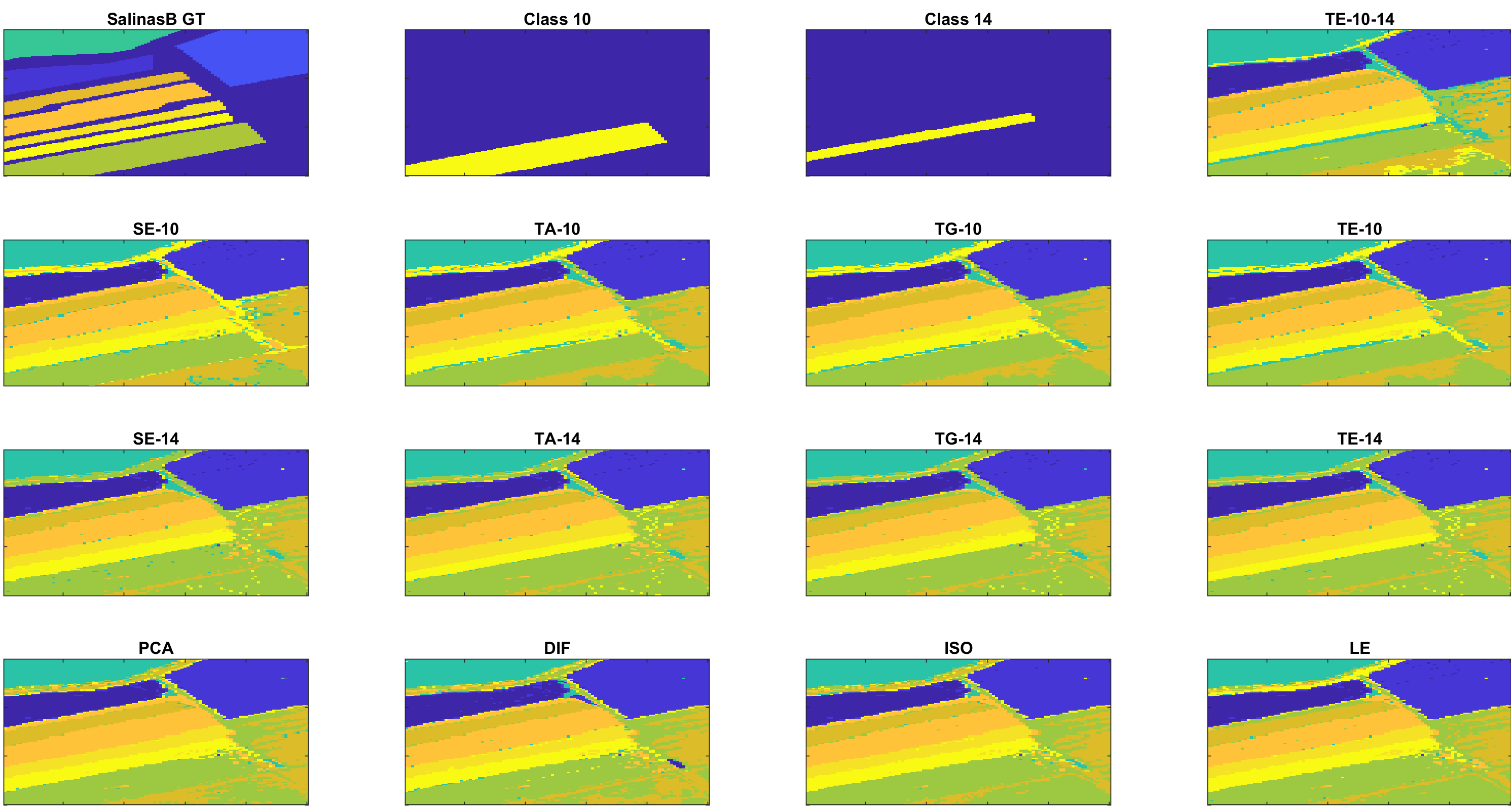}
	\caption{\footnotesize{Classification map of Salinas B}}
	\label{sup:figSalinasB_map}
\end{figure}
And finally on the grouped dataset where similar classes have been merged.
\begin{figure}[H]
	\centering
	\includegraphics[scale = .19 ]{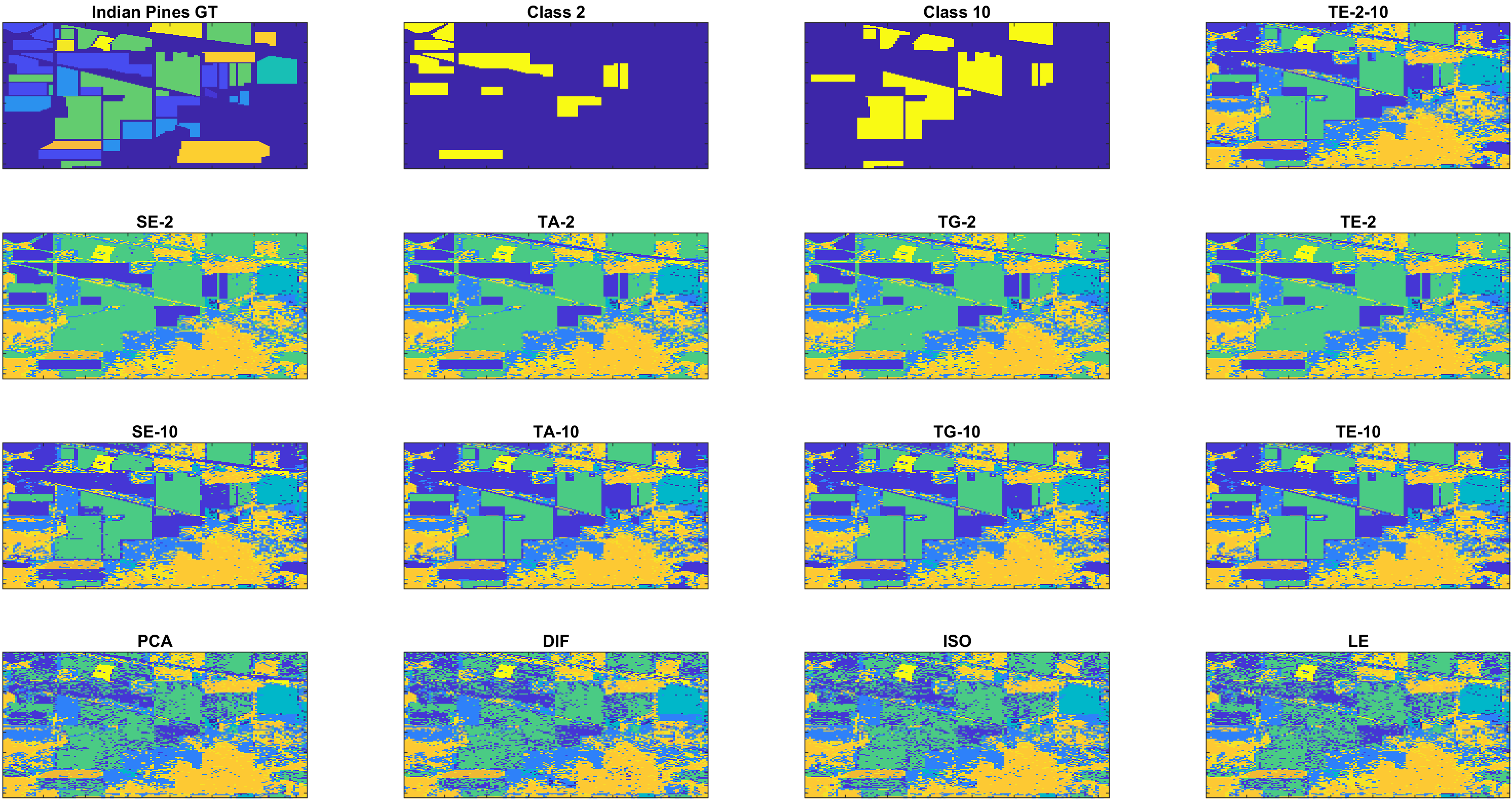}
	\caption{\footnotesize{Classification map of Grouped Indian Pines}}
	\label{sup:figPinesG_map}
\end{figure}

\begin{figure}[H]
	\centering
	\includegraphics[scale = .19 ]{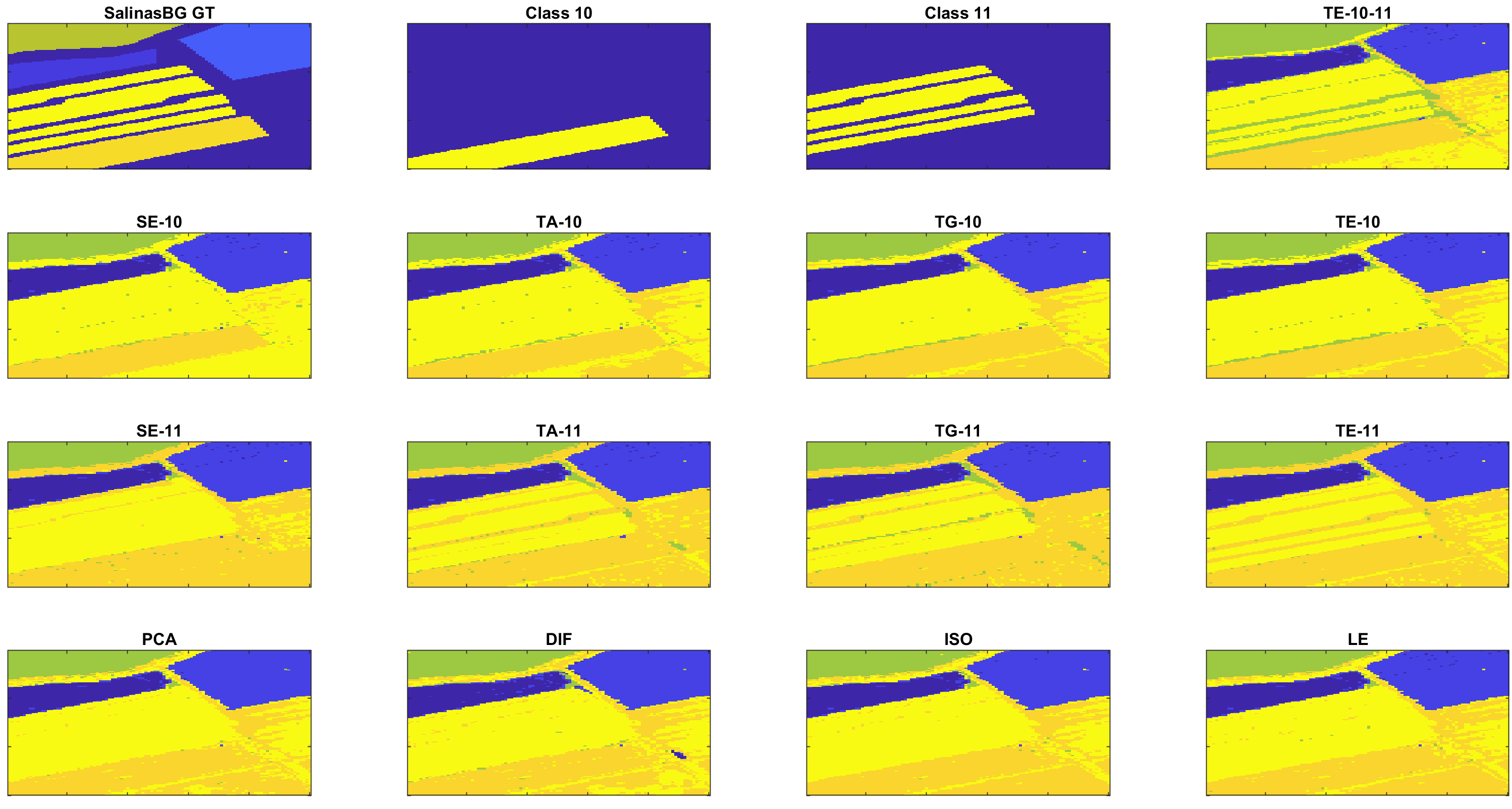}
	\caption{\footnotesize{Classification map of Grouped Salinas B}}
	\label{sup:figSalinasBG_map}
\end{figure}

\end{document}